\documentclass{article}

\pdfoutput=1

\usepackage{PRIMEarxiv}

\usepackage[utf8]{inputenc} 
\usepackage[T1]{fontenc}    
\usepackage{hyperref}       
\usepackage{url}            
\usepackage{booktabs}       
\usepackage{amsmath,amsfonts}       
\usepackage{nicefrac}       
\usepackage{microtype}      
\usepackage{lipsum}
\usepackage{fancyhdr}       
\usepackage{graphicx}       
\usepackage{mathtools} 
\usepackage{booktabs} 
\usepackage{tikz} 
\usepackage{algpseudocode, algorithm}
\usepackage{wrapfig}
 \usepackage[stable]{footmisc}
 
\newcommand{\WA}{\mathbb{E}[W_A]}
\newcommand{\WB}{\mathbb{E}[W_B]}

\newcommand{\textrmrm}{\textnormal}
\newcommand{\textrmit}{\textnormal\textit}
\newcommand{\textrmtt}{\textnormal\texttt}
\newcommand{\textrmbf}{\textnormal\textbf}

\usepackage{amsfonts,amsmath,bm,amsthm}

\DeclareMathOperator*{\argmax}{arg\,max}

\newtheorem{theorem}{Theorem}
\newtheorem{lemma}{Lemma}

\makeatletter
\renewcommand{\ALG@beginalgorithmic}{\small}
\makeatother

\usepackage{natbib} 
\title{DiSC: Differential Spectral Clustering of Features
}

\author{
  Ram Dyuthi Sristi, Gal Mishne \\
  UC San Diego \\
  La Jolla, CA, USA\\
  \textrmtt{\{rsristi, gmishne\}@ucsd.edu} \\
   \And
  Ariel Jaffe \\
  Hebrew University \\
  Jerusalem, Israel\\
  \textrmtt{ariel.jaffe@mail.huji.ac.il} \\
}

\begin{document}
\maketitle

\begin{abstract}
  Selecting subsets of features that differentiate between two conditions is a key task in a broad range of scientific domains. In many applications, the features of interest form clusters with similar effects on the data at hand. To recover such clusters we develop DiSC, a data-driven approach for detecting groups of features that differentiate between conditions. For each condition, we construct a graph whose nodes correspond to the features and whose weights are functions of the similarity between them for that condition. We then apply a spectral approach to compute subsets of nodes whose connectivity differs significantly between the condition-specific feature graphs. On the theoretical front, we analyze our approach with a toy example based on the stochastic block model. We evaluate DiSC on a variety of datasets, including MNIST, hyperspectral imaging, simulated scRNA-seq and task fMRI, and demonstrate that DiSC uncovers features that better differentiate between conditions compared to competing methods.
\end{abstract}

\section{Introduction}

Detecting variables or features that separate between two or more conditions is a critical task in many scientific domains. Often, the separation between conditions is caused by a large number of strongly dependent features that form one or more differentiating clusters. Recovering those clusters can provide insights into the data and its underlying mechanisms. 
Strongly dependent differentiating features are common, for example, in computational biology. In transcriptome analysis, the data consists of the gene expression of multiple cells. Often, the cells correspond to two or more biological conditions, such as various cell types, the existence of a particular disease, response to medical treatment, or two-time points in an evolutionary process \citep{zhao2021detection}. 
Uncovering groups of differentiating genes (also known as \textrmit{pathways}) may significantly contribute to our understanding of biological processes happening in one condition and not in the other. Differential feature grouping has various other interesting applications in neuroscience~\citep{xiao2018alternating}, computer vision and machine learning \citep{daudt2018urban}. In all applications, detecting groups that are significant for differentiating between conditions adds interpretability to the otherwise black box models.

Identifying differential features has been addressed in various contexts. Here we briefly review two main settings: 1) grouping of features with a similar contribution to a prediction task, and 2) data fusion - identifying shared features in datasets from different sources or views. 
In the context of feature selection, regularization based methods are a popular approach. Most of these methods assume a linear regression model over the features to predict a function over the data points, e.g., class label. To select only meaningful features, a regularizer is added to the fidelity term. 
For imposing a sparse solution, 
classical LASSO relaxes the $l_0$ norm  of the coefficient vector with its convex surrogate $l_1$ \citep{elad2010sparse,tibshirani1996regression}. 
Alternatively, a recent approach aims to  approximate the $l_0$ regularizer via stochastic gates \citep{yamada2020feature}.
The sparsity induced by these methods makes them problematic for detecting groups of dependent features since they typically yield only a small number of features out of a large dependent group.

To address the case of correlated features, various extensions of LASSO were developed.
Clustered LASSO \citep{she2008sparse} promotes sparsity among the coefficients and equal significance to non-zero coefficients in the same group.  
Ordered Weighted L1 regularization (OWL) \citep{bogdan2013statistical} penalizes the coefficients in the order of their magnitude, i.e., the higher the absolute value of the coefficient, the higher the penalty that will be imposed. This induces equal coefficients to the correlated features \citep{figueiredo2016ordered}. 
Elastic Net \citep{zou2005regularization} uses a combination of $l_1$ and $l_2$ regularizers which creates a feature grouping effect. Cluster Elastic Net \citep{witten2014cluster} assumes that each feature belongs to one of k-distinct clusters and uses a clustering penalty, along with LASSO, which minimizes the sum of pairwise distances between the associations of features with the prediction within each cluster. 
A method that combines feature selection and grouping was also developed in \cite{bondell2008simultaneous}.
A few methods explicitly use the correlation between the features in the regularization terms. \cite{li2018graph} estimates the covariance between features from which they form a graph. The graph Laplacian matrix is then used as a quadratic regularizer function.

This line of work typically addresses a linear regression model.  However, in many settings the dependency on the important features is highly non-linear. This is typically the case, for example, in most discrete settings such as classification and clustering. 
Furthermore, most of these methods do not separate different groups of correlated features that are significant in a classification or regression problem. In many settings, such a separation is important in order to gain insight to different sources of the variability in the data. 

A different approach related to this work is feature selection via  discriminant subspace analysis methods such as linear discriminant analysis \citep{song2010feature,sharma2014feature} and the Fukunaga-Koontz transform \citep{fukunaga1970application, ochilov2006fukunaga}. Here, the primary goal is  to identify the subspace that best discriminates the classes according to a given criterion. For example, the classic Fischer criterion \citep{fisher1936use} maximizes the ratio between inter-class variance and  intra-class variance. 
Recent works applied similar methods in the context of non-negative and Boolean matrix factorization \cite{gupta2014matrix,hess2017c} .
Some metric learning approaches~\citep{moutafis2017}, which take into account triplet relationships between points, learn a linear transform on the features so that the distance in the projected space better separates the classes. 
However, these approaches do not explicitly perform feature selection to identify groups of correlated features that separate the classes.

In this paper we develop a data-driven method to detect and group relevant features that does not rely on a specific model. We consider the classification setting, where our aim is to identify class-specific information in the feature space that distinguishes between different conditions. To that end, we develop DiSC, a data-driven method to reveal groups of differential features. Our approach consists of two main steps: 
first, for each class, we compute a class-specific graph, whose nodes correspond to the features of that class separately. 
Next, we apply a spectral approach to obtain an embedding of the features that identifies groups of features whose connectivity differs between the graphs. 
Our contributions in this paper are: 1) develop a nonlinear spectral approach on the feature space to identify groups of correlated features that distinguish between datasets, 2) our solution is non-symmetric such that we identify class-specific differential features, 3) we provide theoretical analysis in the setting of a stochastic block model underlying the features. 
Our method is based on removing the significant features of one class from the subspace spanning the significant features of a second class. 
These class-specific subspaces are found in a nonlinear manner by \emph{constructing a graph on the features} of each class separately from which we calculate a nonlinear embedding~\citep{coifman2006diffusion}. 

Our spectral approach is related to recent papers that address the challenges of multiview and data fusion. Here, data samples are observed by multiple sets of sensors, and the goal is to identify latent representations of the samples that are shared for all sets. 
The shared information can be recovered via the shared latent space between multiple sources. 
Alternating Diffusion Maps (ADM)~\citep{lederman2018learning} is a nonlinear manifold learning approach to reveal shared latent variables. ~\citet{shnitzer2019recovering} extend this to identifying both common structures and the differences between the manifolds underlying the different modalities. 
In an alternative solution, \cite{lindenbaum2020multi} proposes a kernel and distance metric for diffusion maps on multiview datasets.  
There are numerous other data fusion techniques \citep{meng2020survey}, \citep{dong2009advances} and applications \citep{huang2016comparison}. Most of these techniques require one-to-one correspondence between data samples from different views, e.g., simultaneous recordings from different sources, and they mainly focus extracting a shared subspace or a shared hidden variable of the samples.
On the other hand, in our setting the correspondence is between features in two datasets, and not between samples. Our goal is to uncover 
connectivity patterns that are \textrmit{condition specific}.

The rest of the paper is organized as follows. Sec.~\ref{sec:prob_form} presents the problem formulation. Our main contribution is in Sec.~\ref{sec:our_approach}. 
In Sec.~\ref{sec:graph_cut} we motivate our approach from a graph-cut perspective. The description of the steps of DiSC are provided in Sec. \ref{sec:alg}. In Sec.~\ref{sec:double_sbm} we provide a  theoretical analysis that is based on the stochastic block model (SBM). The performance of DiSC on simulated and real datasets is demonstrated in Sec.~\ref{sec:experiments}. Conclusions and future work are discussed in Sec.~\ref{sec:conclusion}.

\section{Problem formulation} \label{sec:prob_form}

We begin with a formal description of our problem, followed by relevant notation. We consider two datasets $X^A \in \mathbb R^{n_A \times p}$ and $X^B \in \mathbb R^{n_B \times p}$ where the rows in each matrix are the high-dimensional samples and the columns are the features, see illustration in Fig.~\ref{fig:overview}a. Both  datasets have the same $p$ features, such that the feature column in $A$, denoted  
$X^A_{\cdot i} \in \mathbb R^{n_A}$ corresponds to the feature column in $B$, denoted $X^B_{\cdot i} \in \mathbb R^{n_B}$.  
Our goal in this work is to detect one or more \textit{subsets of features} that together differentiate between the two states $A$ and $B$. These subsets can be, for example, groups of genes that participate in some biological pathways, or a subset of brain parcels in fMRI data with similar blood oxygen level dependent (BOLD) activity.

In our approach, we compute two graphs, denoted $G_A,G_B$ with $p$ nodes that correspond to the features of the given data (Fig.~\ref{fig:overview}b). 
We denote by $W_A,W_B$ the weight matrices of the two graphs, whose elements are functions of the similarity between features, as computed by two kernel functions,
\begin{align}\label{eq:kernel_functions}
&K_A(X_{\cdot i}^A,X_{\cdot j}^A): \mathbb R^{n_A} \times \mathbb R^{n_A} \to \mathbb R 
 \hspace{0.8cm} \textrm{ and } \hspace{0.01cm} 
&K_B(X_{\cdot i}^B,X_{\cdot j}^B): \mathbb R^{n_B} \times \mathbb R^{n_B} \to \mathbb R.
\end{align}
These for example can be the RBF kernel $K(x,x')=\exp\{-\Vert x-x' \Vert^2 / \epsilon^2\}$.
Our underlying assumption is that differences between the two states $A,B$ are expressed as differences in connectivity between the two graphs $G_A$ and $G_B$. For example, a pair of features $i,j$ may be insignificant in state $A$ and highly significant and dependent in $B$. This will imply,
\[
K_A(X_{\cdot i}^A,X_{\cdot j}^A) \approx 0 \hspace{0.3cm} \textrm{ and } \hspace{0.3cm} K_B(X_{\cdot i}^B,X_{\cdot j}^B) > 0.
\]
If there is a subset of $l$ correlated features $i_1,\ldots,i_l$ that are strongly dependent in $B$ but not in $A$, the resultant nodes in $G_B$ will form an independent component with dense connections among the nodes. However, this independent component will not exist in $G_A$, see illustration in Figure \ref{fig:overview}. Thus, the task of obtaining sets of 
significant features boils down to identifying the independent components that appear in one graph but not the other.

\section{DiSC} \label{sec:our_approach}
In this section we describe our data-driven approach for detecting differential groups of features. In section \ref{sec:graph_cut} we present a \textrmit{graph-cut} perspective to the problem of group feature selection. In section \ref{sec:alg} we describe our approach. In section \ref{sec:double_sbm} we analyze our approach based on the stochastic block model.

\subsection{A graph cut perspective}\label{sec:graph_cut}
Given a graph $G$ with $n$ nodes and its associated weight matrix $W \in \mathbb{R}^{n \times n}$, 
the minimum-cut of $G$ is the 
minimum, over all possible partitions $\alpha$ and $\beta$, of the sum of the edge weights $\sum_{i \in \alpha,j \in \beta} W(i,j)$. 
This task is strongly related to the spectral clustering algorithm. 
For completeness, we begin with a brief description of this relation. For a more thorough review see~\cite{von2007tutorial}.

A variation of the minimum-cut, designed to avoid highly imbalanced partitions, is the ratio-cut
\begin{equation}\label{eq:ratio_cut}
\textrm{Rcut}(\alpha,\beta) = \sum_{i\in \alpha,j \in \beta} W_{ij} \left( \frac{1}{|\alpha|}+ \frac{1}{|\beta|}\right),
\end{equation}
where $|\alpha|,|\beta|$ denote the subset size.
This can be formulated as 
\begin{equation}\label{eq:ratio_cut_vector}
\textrm{Rcut}(\alpha,\beta) = f^T L_u f \hspace{0.6cm} \textrm{ where } f_i = \begin{cases}
\sqrt{\frac{|\beta|}{|\alpha|}} & i \in \alpha \\
-\sqrt{\frac{|\alpha|}{|\beta|}} & i \in \beta
\end{cases}.
\end{equation}
and $L_u$ is the unnormalized graph Laplacian matrix given by $D-W$, where $D$ is a diagonal matrix containing in the diagonal the degrees of each node. 
Note that $ f$ is a non-binary indicator vector and if  $\alpha,\beta$ are non-empty, then $f$ is orthogonal to the constant vector with $\| f\|_2^2=n$.  
Minimizing Eq. \ref{eq:ratio_cut_vector} over all partitions is  a discrete optimization problem. To avoid it, one can relax the discrete requirement over the elements of $f$ in Eq. \ref{eq:ratio_cut_vector}, while still maintaining the constraints $f^T 1 = 0$ and $\| f\|_2^2 = n$. This relaxation yields a simple spectral solution, where the nodes are partitioned according to the sign of the graph Laplacian eigenvector corresponding to the second smallest eigenvalue (a.k.a Fiedler vector). 

\begin{figure}[ht]
    \centering
    \includegraphics[width=0.95\linewidth]{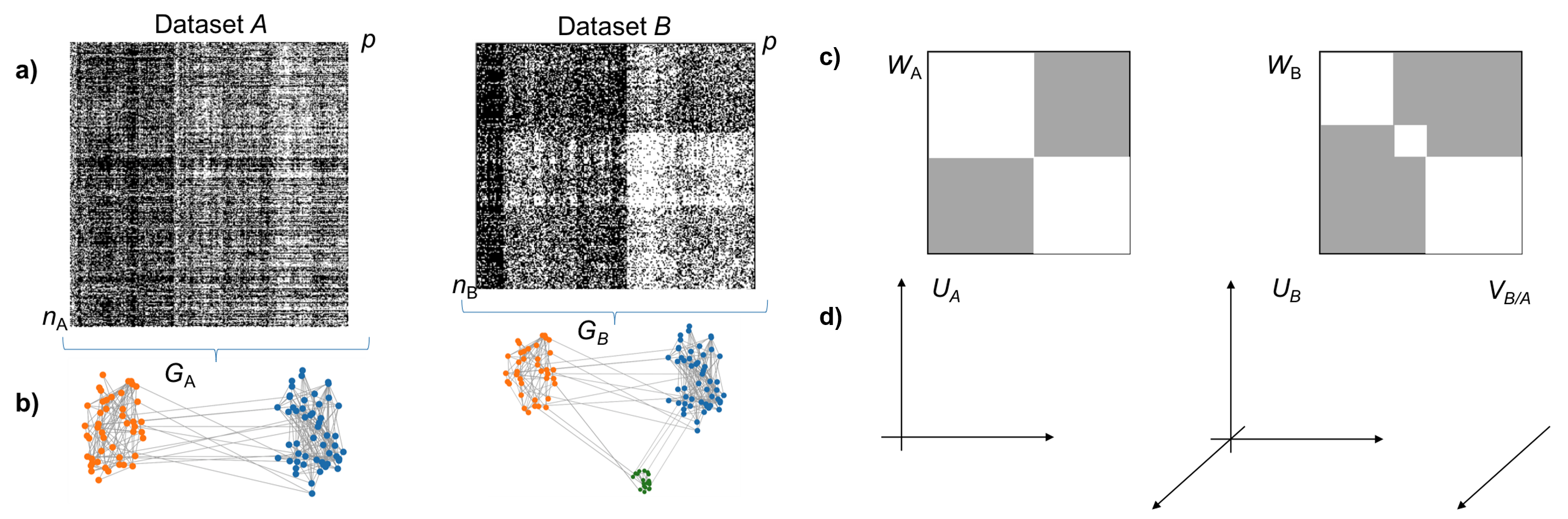}
    \caption{DiSC overview. a) Datasets $A$ and $B$ share the same $p$ features. b) Constructing graphs $G_A$ and $G_B$ on the feature space for each dataset, we aim to find which nodes have different connectivity pattern between the two datasets. c) The weight matrices each follow a stochastic block model, where in $B$ a subset of of the nodes form a separate block.
    d) Subspaces spanned by the eigenvectors of $P_A$ and $P_B$, and the differential feature between $B$ and $A$.}
    \label{fig:overview}
\end{figure}

The minimum-cut criterion is designed to attain densely connected components that have low connectivity with the rest of the graph. In our work, however, densely connected components that appears in both graphs are of no interest, as they do not differentiate between the two states. Rather, we would like to detect a partition $(\alpha,\beta)$, say on $G_B$, 
that has two properties: 1)   Rcut$(\alpha,\beta)$, on $G_B$ is minimized, which indicates that $\alpha$ and $\beta$ are independent connected components in $G_B$; 2) 
To avoid detection of components that exist  both in $G_A$ and $G_B$, Rcut$(\alpha,\beta)$ on $G_A$ is lower bounded by some constant.
This double objective can be formulated by
\begin{equation}\label{eq:double_objective}
\min_{f} f^T L_B  f  \qquad \textrm{s.t.} \quad f^T L_A  f \geq \gamma, \quad \|f\| = 1.
\end{equation}
The constraint  $\|f\|=1$ is added for uniqueness, since $L_A,L_B$ are rank deficient.  
Relaxing the discrete requirement in Eq. \ref{eq:ratio_cut_vector}, the solution to Eq. \ref{eq:double_objective} is given by the generalized eigenvalue problem $L_B f = \mu L_A f$. If $L_A$ is invertible, the solution to this is given by the eigenvalue problem $L_A^{-1} L_B f=\mu f$. But if $L_A$  is not invertible, one numerical trick is to strengthen the diagonal of $L_A$ to make it a full rank matrix. Then the solution is given by the eigenvalue problem, 
\begin{equation}\label{eq:generalized_eigenvector}
(L_A + \epsilon I)^{-1} L_B f = \mu  f.   
\end{equation}
Motivated by the above derivation, in the next section we present our spectral approach for group feature discovery. 
In the appendix we present a similar derivation that is based on the normalized cut, whose relaxation relates to the eigenvectors of the random walk graph Laplacian.

\subsection{Spectral approach for group differential feature extraction}
\label{sec:alg}
\looseness=-1
In practice, our approach is based on the spectral decomposition of the two random-walk Laplacian matrices $P_A=D_A^{-1}W_A$ and $P_B=D_B^{-1}W_B$, where $D_A$ and $D_B$ are the degree matrices of $A$ and $B$ respectively, rather than the unnormalized graph Laplacian.
Let $U_B \in \mathbb R^{p \times d_B}$ be a matrix containing the $d_B$ leading right eigenvectors of $P_B$. 
These vectors, scaled by their corresponding eigenvalues, are the diffusion vectors of the two graphs~\citep{coifman2006diffusion}. A vector containing the $i$-th  element of the columns of $U_A$ is defined as the \textrmit{diffusion map representation} of the $i$-th node.
Here, we define a second operator $Q_B$, that forms a projection matrix onto the complementary subspace of $U_B$:
\begin{equation}\label{eq:projection}
Q_B = I - U_B (U_B^T U_B)^{-1} U_B^T.
\end{equation}

\looseness=-1
We define the \textrmit{differential vectors} of $G_A$, denoted $(V_A)_i$ and their significance level $(\sigma_A)_i$ via
\begin{align}
     (V_A)_i  = 
        \argmax_{\textrmrm{dim} E=i} \min_{v\in S(E)} ||P_AQ_Bv||_2.  \hspace{0.5cm} \textrm{and} \hspace{0.5cm}
     (\sigma_A)_i   = ||P_AQ_B(V_A)_i||_2. \label{eq:differentiating_vectors}
\end{align}
where $E$ is a subspace in $\mathbb R^{p}$ and $S(E)$ denotes unit Euclidean sphere in $E$. The relation between the diffusion vectors $U_A$ and the differentiating vectors $V_A$ is similar in nature to the relation between the outcome of the min-cut criterion in Eq. \ref{eq:ratio_cut_vector} and the double objective in Eq. \ref{eq:double_objective}. The diffusion vectors in $U_B$ highlight significant processes that underlie state $B$. However, we are only interested in processes that differentiate between state $A$ and $B$. These are highlighted in the differential vectors $V_B$.
Thus, our approach for detecting differences between graphs is similar to the solution of the double objective criterion given in Eq.~\ref{eq:generalized_eigenvector}. However, to increase the stability of our approach we make two changes:
1) replace the inverse matrix in Eq.~\ref{eq:generalized_eigenvector} with the projection matrix $Q_B$ to the complementary subspace, thus avoiding potential complications and additional constraints required for the inverse computation.
2) Instead of the right eigenvectors, we compute the leading singular vectors.

Note that the operator $P_AQ_B$ whose singular vectors are used to attain differential features is not symmetric in $A$ and $B$. The differential vectors in $V_A$ 
highlight components that are significant in $A$ but not in $B$. The differential vectors $V_B$ are attained from the singular vectors of $P_BQ_A$.

\paragraph{Differential vectors for downstream analysis} 

The differential vectors $V_A,V_B$ and their scores $\sigma_A,\sigma_B$ can be used in several ways. Here, and in the experimental section, we apply the vectors in one of two applications: (i) Detecting of differential subsets of features. This can be done by performing $k$-means clustering over the rows of the matrix $V_A$; (ii) Computing differential meta-features via $V_A^T X_A$ and $V_B^T X_B$. In section \ref{sec:experiments} and supplementary material we provide several examples for using both in downstream analysis (e.g., clustering, classification). 
In the next section we derive  a theoretical justification for application (i) for our approach that is motivated by the stochastic block model. Algorithm~\ref{alg} summarizes the steps of the DiSC algorithm. Further details about the computation of the graph and discussion on choice of hyperparameters is given in App.~\ref{app:hyperparam}.

\begin{algorithm}[ht]
\caption{DiSC}
\begin{flushleft}
\textrmbf{Input:} 
\begin{tabular}[t]{ll}
Datasets $X^A$ and $X^B$ \\
Two kernel functions $K_A(\cdot,\cdot)$ and $K_B(\cdot,\cdot)$
\end{tabular} \\
\textrmbf{Output:} Subsets of differentiating features $V_A,V_B$
\end{flushleft}
\vspace{-0.2cm}
\begin{algorithmic}[1]
\State Compute two graphs $G_A$ and $G_B$ on the columns of $X^A$ and $X^B$ with weights given by \ref{eq:kernel_functions}.
\State Compute the random walk matrix, $
P_A = D_A^{-1}W_A, P_B = D_B^{-1}W_B.
$
\State Calculate $U_A,U_B$, the leading right eigenvectors of $P_A,P_B$.
\item Compute the projection matrices $Q_A,Q_B$ via Eq. \ref{eq:projection}. 
\State Compute differential vectors $V_A$ and $V_B$ via Eq. \ref{eq:differentiating_vectors}.
\State Compute significance levels  $\sigma_A$ and $\sigma_B$ via Eq. \ref{eq:differentiating_vectors}.
\State optional: Perform k-means over the rows of $V_A$, and $V_B$.
\end{algorithmic}
\label{alg}
\end{algorithm}

\subsection{Two stochastic block models}\label{sec:double_sbm}

The stochastic block model (SBM) has received a lot of attention due to its significant role in obtaining theoretical guarantees for community detection. In this setting, the individual community members are modeled as nodes in a random graph $G$. 
The edge weights $W_{ij}$ are sampled according to a Bernoulli distribution, with
\begin{equation}\label{eq:sbm}
\Pr(W_{ij} = 1) = 
\begin{cases}
p &  \text{  if $i,j$ belong to the same community} \\
q &  \text{  otherwise}.
\end{cases}
\end{equation}
Usually, one assumes $p>q$ such that the connectivity within a block is stronger than the connectivity between blocks, See ~\citep{abbe2015exact} for further details. 

In our setting, we model the features in both states by two random graphs $G_A$ and $G_B$.
We consider a toy problem in which the features that differentiate two states $A$ and $B$ have different graph connectivity in $G_A$ and $G_B$. 
In graph $G_A$ the nodes are partitioned into two communities of sizes $l$ and $l+s$, with $s < l$. In contrast, the graph $G_B$ has three communities, denoted $\alpha,\beta$ and $\gamma$. Community $\alpha$ is equal to the first community in $G_A$, while $\beta,\gamma$ are a partition of the second into two groups of size $l$ and $s$.

Our goal is to detect the elements in $\gamma$ via a spectral approach whose steps are similar to Algorithm \ref{alg}. 
For simplicity of exposition, instead of using the graph Laplacian matrix as in step 2 of the algorithm, we use the symmetric weight matrices $W_A,W_B$. Our goal in this analysis is to provide insight into our ability to recover differentiating groups such as $\gamma$ via Algorithm \ref{alg}. Our main parameters of interest are the size of the differentiating group $\gamma$, and the ratio between the size of $\gamma$ and the size of the original block. Elements that depend on other parameters of the model (e.g. $p$ and $q$) are referred to as constants.
Our derivation consists of three main steps. The proofs of our results are given in Appendix \ref{app:proofs}. In addition, appendix \ref{sec:validation_theory} presents numerical results that validate the bound in lemma \ref{lem:concentation_v_gamma} and test its optimality.  
We note that all matrix norms (i.e. $\|X\|$) in the following section and relevant appendices are the spectral norm. 

\noindent \textbf{Step 1:} Here we address the non-random matrix $\WB$ where $\mathbb{E[\cdot]}$ is the expectation operator. 
Let $v_\alpha,v_\beta,v_\gamma$ be the leading three eigenvectors of $\WB$ and by $e_\gamma$ a binary indicator vector with elements $(e_\gamma)_i=1$ if $i \in \gamma$.  
\begin{lemma}\label{lem:expected_wb}
The distance between $v_\gamma$ and $\frac{1}{\sqrt{s}}e_\gamma$ is bounded by
\[
\Big\|v_\gamma - \frac{1}{\sqrt{s}} e_\gamma\Big\|\leq C_1(p,q)\sqrt{\frac{s}{l}}.
\]
The eigenvalue corresponding to $v_\gamma$ is larger than $(p-q)s$.
\end{lemma}
Lemma \ref{lem:expected_wb} shows that if the ratio $s/l$ is small enough, one can recover the elements in $\gamma$ by applying a threshold to $v_\gamma$.
However, in practice we only have access to $W_B$ and $W_A$.
In the next step we bound the difference between $v_\gamma$ and the eigenvector we compute by our spectral approach. 

\noindent \textbf{Step 2:}
Let $u_\alpha,u_\beta$ be the two leading eigenvectors of 
the random weight matrix $W_A$, and let $Q_{W_A} = I-u_\alpha u_\alpha^T - u_\beta u_\beta^T$.
\begin{lemma}\label{lem:concentation_v_gamma}
Let $\tilde v_\gamma$ denote the leading eigenvector of $Q_{W_A} W_B Q_{W_A}$. Then, 
\[
\|\tilde v_\gamma - v_\gamma\| \leq C_2(p,q)\frac{\sqrt{l}}{s} + C_3(p,q) \sqrt{\frac{s}{l}}. 
\qquad \text{w.p} \qquad 1-\exp(-l).
\]
\end{lemma}

\noindent \textbf{Step 3:} 
Observing the two lemmas, we see that there is a tradeoff concerning the size $s$ of the differentiating group $\gamma$. 
On the one hand, Lemma~\ref{lem:expected_wb} shows that  having a small value for $s$ makes the element more distinguishable in 
$v_\gamma$ and thus easier to detect. On the other hand, if $s < \sqrt{l}$ then the computed vector $\tilde v_\gamma$ might be too noisy to actually detect the relevant features. 

Combining the two lemmas, we conclude with the following theorem.
\begin{theorem}\label{thm:guarantee}
We assume that $s,l$ are large s.t. $s,l \gg \max_i(C_i)$, and $s = l^\alpha$ with  $0.5 < \alpha <1$.  We apply a threshold to  $\tilde v_\gamma$ 
to determine the elements of $\gamma$. The relative number of errors $N_\epsilon/s$ is bounded by
\[
\begin{cases}
C_2(p,q) l^{1-2\alpha} & 0.5< \alpha\leq 2/3 \\
C_3(p,q) l^{\alpha-1} & 2/3< \alpha < 1
\end{cases}.
\qquad \text{w.p} \qquad 1-\exp(-l).
\]
\end{theorem}
\begin{proof}[Proof of Theorem \ref{thm:guarantee}]
Let $C_4(p,q) = C_1(p,q)+C_3(p,q)$. combining Lemmas \ref{lem:expected_wb}, \ref{lem:concentation_v_gamma} with the triangle inequality, and assuming $ s = l^\alpha$  yields the inequality
\[
\Big\|\tilde v_\gamma - \frac{1}{\sqrt{s}}e_\gamma\Big\|^2 \leq C_2(p,q)^2 l^{(1-2\alpha)} +C_4(p,q)^2 l^{(\alpha-1)}+ C_2(p,q)C_4(p,q)l^{-\alpha/2} .
\]
For $0.5< \alpha < 2/3$, the dominant term is $C_2(p,q)^2 l^{1-2\alpha}$. 
We recover the elements in $\gamma$ by setting a threshold of $\frac{1}{2\sqrt{s}}$ to $\tilde v_\gamma$.
Each misclassified element contributes at least $1/(4s)$ to the squared $l_2$ distance between $\frac{1}{\sqrt{s}}e_\gamma$ and $\tilde v_\gamma$. Thus, for $0.5< \alpha \leq 2/3$ the relative number of errors $N_\epsilon/s$ is bounded with high probability by
\[
\frac{N_\epsilon}{s} \leq \Big\|\tilde v_\gamma - \frac{1}{\sqrt{s}}e_\gamma\Big\|^2 \times (4s)/s \leq 4C_2(p,q)^2 l ^{1-2\alpha}
\]
A similar derivation can be done for $2/3<\alpha<1$.
\end{proof}

\section{Experiments \footnote{Code to reproduce the results for Sections \ref{sec:toy_problem} and \ref{sec:mnist} is available in \texorpdfstring{\url{https://github.com/Mishne-Lab/DiSC}}{https://github.com/Mishne-Lab/DiSC}}} \label{sec:experiments}

\subsection{Toy problems }\label{sec:toy_problem}
We demonstrate the usefulness of our methodology in three toy datasets which present different scenarios of feature correlation patterns: 
\looseness = -1
1) a subset of features that is uncorrelated in one dataset becomes correlated in the second 2) a subset of features that are correlated in one dataset are divided into two correlated groups in the second. 3) The third problem demonstrates generalizing DiSC to more than two datasets. In all experiments the kernels $K_A,K_B$ are the RBF kernel with an adaptive bandwidth. The choice of the all the hyperparameters (bandwidth, $d_A$ and $d_B$) are discussed in Appendix~\ref{app:hyperparam}. 

\begin{figure}[th]
    \centering
    \includegraphics[width=\textwidth]{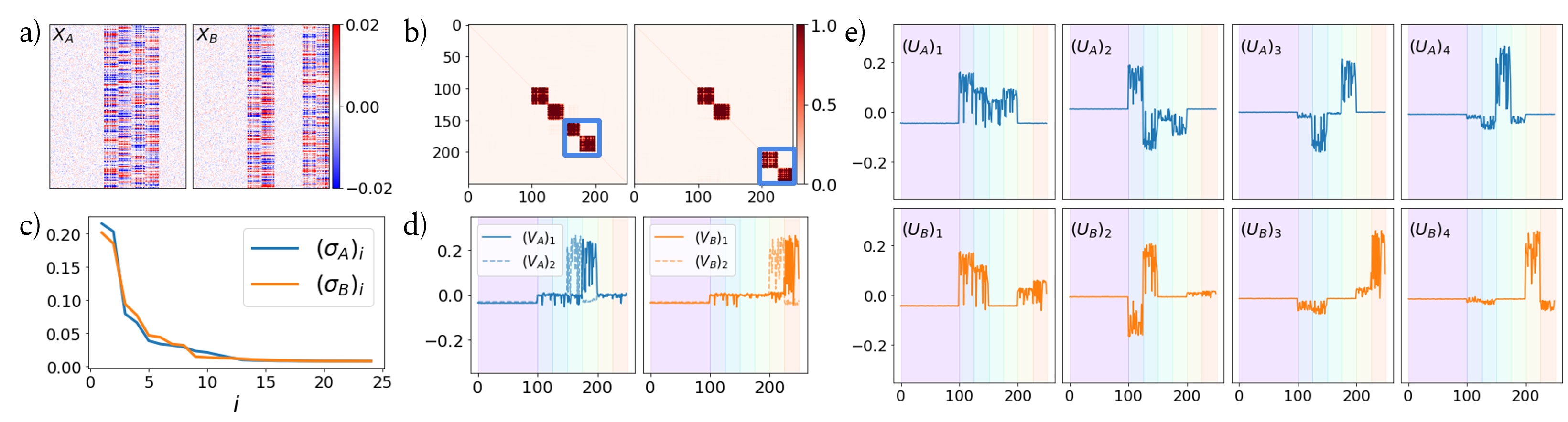}
    \vspace{-0.7cm}
    \caption{a) Random samples from $X^A$ and $X^B$. b) Feature correlation matrices of $X^A$ (left) and $X^B$ (right). c) Significance of the differential vectors of $X^A$ (blue) and $X^B$ (orange). d) Top two differential vectors. e) Top four diffusion maps of the features of  $X^A$ (blue) and $X^B$ (orange). }
    \label{fig:toy1}
\end{figure}
\paragraph{Identifying newly connected features}
We generated two datasets $X^A$ and $X^B$ with $p=250$ features and $n_A=n_B=10,000$ samples from a Gaussian mixture model (Fig.~\ref{fig:toy1}(a)), whose feature correlation is shown in Fig.~\ref{fig:toy1}(b). In the two datasets, the first 100 features are i.i.d samples from a normal distribution and the next 50 features are sampled from two Gaussian distributions with low rank covariance matrices. Features 151-200 in $X^A$ are sampled from two other Gaussian distributions with low rank covariance matrices, whereas in $X^B$, they are independent noise. Features 201-250 in $X^A$ are independent noise,  whereas in $X^B$, they are sampled from two other Gaussian distributions with low rank covariance matrices. 

Our goal is to identify the features 151-200 as the features that distinguish $X^A$ from $X^B$, and vice versa for features 201-250. We apply DiSC with $d_A=d_B=20$. The first two significance values of the differential features are high and then significance drop drastically (Fig.~\ref{fig:toy1}(c)), indicating that the first two differential vectors are important in each dataset. These top two differential features for $X^A$ and $X^B$ are shown in Figure~\ref{fig:toy1}(d). Clearly, $(V_A)_1$ and $(V_A)_2$ highlight the features between 151-200 and more precisely, the two Gaussian mixtures are separately represented in each of these vectors. Similarly, $(V_B)_1$ and $(V_B)_2$ highlight the features between 201-250 and these represent the other two Gaussian mixtures. On the other hand, the diffusion maps eigenvectors, $U_A$ and $U_B$, captures all the connected components and not just the differential features as unique groups, as shown in Fig.~\ref{fig:toy1}(e).

\paragraph{Identifying subsets of connected features} \label{par:toy_problem2}
We generated two datasets $X^A$ and $X^B$ with $p=200$ features and $n_A=n_B=10,000$ with correlation between the features as shown in Fig.~\ref{fig:toy2}(a). In both the datasets, the first 100 features are correlated. In $X^A$, all the remaining features form a second correlated group, yielding two connected components in the feature space. In $X^B$ the remaining features are composed of two groups of correlated features, namely, feature indices 101-125 and 126-200. 

Here the goal is to identify these smaller subsets of correlated features in $X^B$, as this signifies an increase in the dimensionality of correlation structure compared to $X^A$. This also means that $X^B$ has richer feature information than $X^A$, therefore we have to identify that there are no differential features in $X^A$.

We use the significance level associated with the differential vectors to determine if the differential vectors are meaningful. We apply DiSC with $d_A=d_B=20$. Fig.~\ref{fig:toy2}(c) shows the differential vectors of $X^A$ and $X^B$ and the corresponding significance levels are shown in Fig.~\ref{fig:toy2}(b). The significance level associated with the differential vectors of $X^A$ is negligible compared to that of $X^B$, indicating that feature information in $X^A$ is contained in $X^B$. In addition, the difference between the significance of the first two differential vectors of $X^B$ is large. Therefore, the first differential vector of $X^B$ captures the major difference, and it groups feature indices 101-125 and 126-200 with positive and negative values respectively (Fig.~\ref{fig:toy2}(c)).  In contrast, diffusion maps captures the two connected components in $X^A$ and three connected components in $X^B$, and not just the differential components (Fig.~\ref{fig:toy2}(d)).     

\begin{figure}
    \centering
    \includegraphics[width=\textwidth]{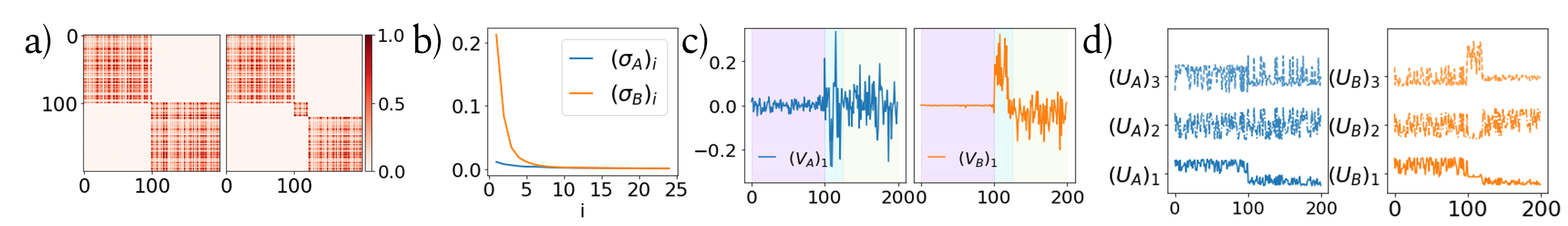}
     \vspace{-0.7cm}
    \caption{ a) Correlation matrices. b) Significance levels. c)~ Differential features. d) Diffusion maps. }
    \label{fig:toy2}
\end{figure}

\subsection{MNIST}\label{sec:mnist}
The MNIST dataset \citep{lecun-mnisthandwrittendigit-2010} consists of images of hand written digits from 0 to 9, each of dimension $28\times28$ pixels. It has 60k training samples and 10k testing samples. We used Algorithm~\ref{alg_multiple}, DiSC on multiple datasets, on the MNIST dataset to extract differential features for each of the digits: features present in a single digit but are not present in each of the other digits. To have a quantitative metric to measure the significance of these differential features, we develop a classifier for a 10-class classification problem. For this, we use K-means clustering on the differential features for each digit (K=10) thus deriving clusters of pixels that together differentiate between digits. We then compute a meta-feature by averaging pixel values for each cluster. Logistic regression is performed on these meta features to distinguish between the 10 classes. 

We compare our approach with Diffusion maps, Elastic Net (EN) and Elastic Net - logistic. For diffusion maps, we replaced the differential features with diffusion maps and followed the same procedure. For EN and EN-logistic, we obtained the feature importance vectors by training a classifier to distinguish between the 10-classes. K-means clustering is performed on these feature importance vectors. This cluster information is further used to compute meta-feature and finally for classification. Additionally, the entire data (784 features) is used to obtain the clusters using K-means clustering which are further used to compute classification accuracy following the procedure mentioned above. We consider this as a baseline. The results are tabulated below in Table~\ref{tab:mnist_multiclass}. 

We can see that our proposed method has the best performance, as it can capture the important differential information needed for classification. Since diffusion maps capture both shared and differential features, and entire data has all the information, their performance is poor compared to DiSC. Elastic Net has the least accuracy because it is designed for a regression setup but this is a classification problem. Elastic Net-logistic has better performance over Elastic Net as we use a cross entropy loss function – the one used for classification setup. Further, DiSC has better performance than Entire Data as it is designed to capture the differential features. However, DiSC outperformed EN-logistic as it is designed to extract groups of differential features unlike EN-logistic which just extracts all the differential features as a single group, especially when those features have similar effect on the classification problem.  Note in App.~\ref{appendix:experiments} we include an additional experiment on pairwise classification of MNIST digits which also includes more details about the choice of hyper-parameters for each of these methods.

\begin{table}[h]
       \centering
      \begin{tabular}{|c|c|c|c|c|c|}
      \hline
            & Entire Data          & Diffusion Maps     & EN     & EN-logistic & DiSC (Ours) \\ \hline
            Test Accuracy($\%$ )     & 76.23                 & 87.42            & 69.96        & 82.42  & \textbf{89.75}\\ \hline
      \end{tabular}

       \caption{Classification performance on MNIST using DiSC and other competitive approaches.}
       \label{tab:mnist_multiclass}
   \end{table}

\vspace{-1cm}
\subsection{Hyperspectral imagery}\label{sec:experiments_hyperspectral}
We apply DiSC to a hyperspectral imagery change detection dataset \citep{eismann2007hyperspectral}, consisting of hyperspectral images of a particular scene that were captured during different weather conditions, lighting conditions, and across four different months of the year (August, September, October and November). 
Each hyperspectral image is captured at a resolution of $100\times100$ pixels and 124 different bands. We denote these by $X^\textrm{aug},X^\textrm{sep},X^\textrm{oct} $ and $X^\textrm{nov}$. These images consist of a metal frame, grass and trees in the background, as shown in Fig.~\ref{fig:hyper_original}. 
An additional hyperspectral image, referred to as `October-change' and denoted $X^\textrm{oct-c}$, was captured in October, in which an added object, a \emph{tarp}, was included in the scene, see Fig.~\ref{fig:hyper_original}.
The goal is to detect the tarp as an added object, albeit various other aforementioned different conditions during which the image is captured. 
\begin{figure}
    \centering
    \includegraphics[width=\textwidth]{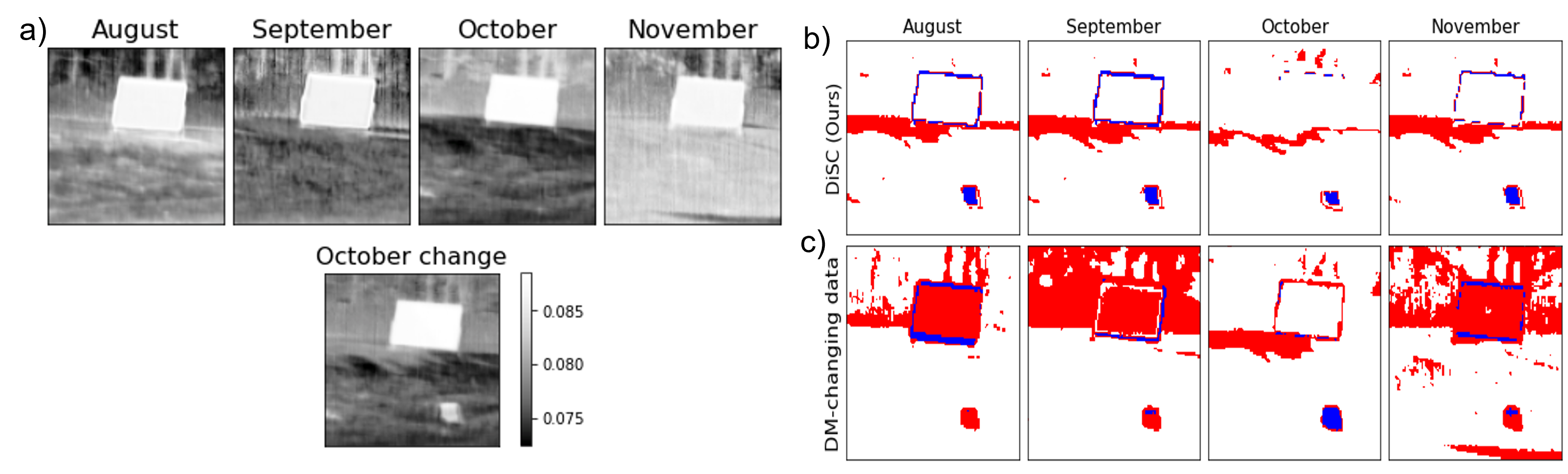}
    \caption{Hyperspectral imaging averaged across channels (a). Three clusters of features (white, red and blue) formed from DiSC (b) and Diffusion Maps for changing data (c).}
    \label{fig:hyper_original}
\end{figure}

We consider $X^A \in \{X^\textrm{aug},X^\textrm{sep},X^\textrm{oct} $,$X^\textrm{nov} \}$ and $X^B=X^\textrm{oct-c}$ and compute the differential vectors for these four pairs of datasets. Note that this is similar to the toy problem presented in Sec.~\ref{par:toy_problem2} and the theoretical analysis 
in Sec.~\ref{sec:double_sbm}, where a group of correlated features in $X^A$ (the pixels belonging to the grass) are divided into two groups of correlated features in $X^B$ (grass and tarp). With the addition of the tarp, the pixels belonging to the tarp remain correlated but their connectivity with the other pixels in the grass is lost. The differential features of $X^B$ identify the tarp, see App.~\ref{appendix:experiments} Fig.~\ref{fig:hyper_oct_octc}.

We pick the top four significant differential vectors of $X^B$ and perform k-means clustering on these with k=3. Figure~\ref{fig:hyper_original}(b) shows these three feature clusters for the four pairs of datasets. The tarp is revealed as a dominant cluster in all months. We compare our results with Diffusion maps for changing data (DM-changing data) \citep{coifman2014diffusion}, a spectral approach designed to capture differences between two conditions, which introduces a distance metric that measures the distance between diffusion maps calculated on a dataset that changes over time. We compute this pixel-wise distance between $X^A$ and $X^B$ and cluster the pixels into 3 groups based on this distance. These clusters are illustrated in Figure.\ref{fig:hyper_original}(c). This approach is much more affected by the acquisition conditions, For example: weather and lighting, than our approach, as it groups additional objects that have not changed along with the tarp. Finally, DiSC performs better than DM-changing data in the presence of added noise (see Appendix~\ref{appendix:experiments}). 
\vspace{-6pt}

\subsection{fMRI}

\begin{wrapfigure}[20]{r}{0.5\textwidth}
\vspace{-20pt}
  \begin{center}
    \includegraphics[width=0.48\textwidth]{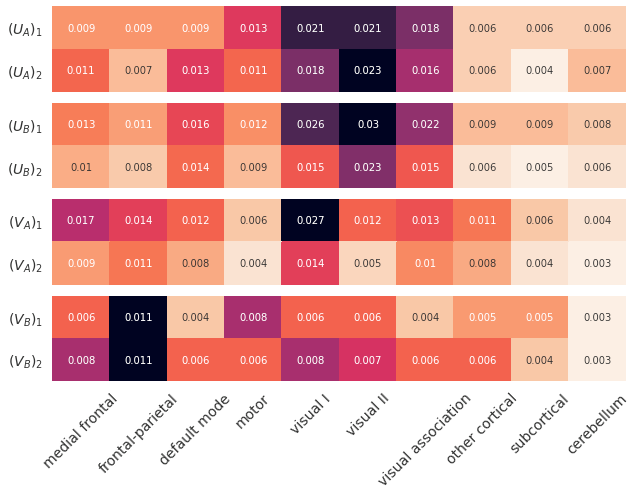}
  \end{center}
  \caption{Canonical correlations of subject-averaged diffusion maps and differential features with 10 canonical brain networks for fMRI data of a 0-back (task $A$) and 2-back (task $B$) working memory tasks.}
  \label{fig:fmri}
\end{wrapfigure}

We assess the performance of DiSC in identifying groups of brain parcels with correlated BOLD activity in a working memory task from the Human Connectome Project~\cite{van2013wu}, where subjects executed interleaved blocks of 0-back and 2-back working memory tasks.
In these tasks, subjects are instructed to monitor a sequence of visual items and to respond whenever a presented item is the same as the one previously presented 2 items ago (2-back) or a predetermined item (0-back). 
This fMRI dataset consists of 515 subjects (subjects with high motion or incomplete data were removed), and a whole-brain, functional atlas~\cite{shen2013groupwise} was used to extract time-courses from $p=268$ brain parcels.
Dataset $X^A$ is composed of all blocks from the 0-back task and $X^B$ is composed of all blocks from the 2-back task.

For each subject we calculated the top 2 diffusion maps eigenvectors as well as the top 2 differential features for each dataset, and average these across all subjects.
Fig.~\ref{fig:fmri} displays the correlation of each of the averaged vectors with indicator vectors for 10 canonical brain networks~\cite{finn2015functional}.
Results show that diffusion maps is mainly correlated with visual networks, while the differential vector for the 0-back task being most correlated with the visual II network. On the other hand, the differential vector for the 2-back task is most correlated with the frontal-parietal network, which has been shown to be predictive of working memory performance~\cite{avery2020distributed}. Thus, as opposed to diffusion maps, DiSC reveals that the 2-back task incorporates more high-level cognitive regions (e.g., prefrontal) compared to the 0-back task which has lower cognitive load~\cite{gao2019hierarchical,gao2021nonlinear}.
\vspace{-6pt}

\section{Discussion and future work}\label{sec:conclusion}
In this paper we introduced DiSC, a spectral approach for finding differential features between two or more datasets. 
We demonstrate the results of our model on various synthetic and real-world datasets and show that DiSC extracts better differential features as compared to the competing techniques. We also show the experimental results on more than two datasets.
One limitation of our method is that it addresses only differences in ``connectivity", or correlation, between features, not the feature values themselves. Another limitation is the choice of the hyperparameters, $d_A$ and $d_B$. High values would result in extracting noise or nuisance features as the differential features, and low values might not detect the essential differential features. 
Finally, the problem of ``redundant" eigenvectors~\citep{dsilva2018parsimonious} arising in spectral clustering and manifold learning may further complicate choosing the correct dimensionality. 
This can be mitigated by using non-redundant eigenvectors~\citep{blau2017non} which we will explore in future work.

\section{Acknowledgements}
Data were provided in part by the Human Connectome Project, WU-MinnConsortium (Principal Investigators: David Van Essen and Kamil Ugurbil;1U54MH091657) funded by the 16 NIH Institutes and Centers that support the NIH Blueprint for Neuroscience Research.
The authors thank Siyuan Gao for preprocessing of the fMRI dataset and for valuable discussions.



\nocite{*}

\appendix

\section{Choice of hyperparameters}
\label{app:hyperparam}
The Disc algorithm requires the following hyperparameters: (i) The bandwidth for the kernel functions $K_A(X^A_{\cdot i},X^A_{\cdot j}), K_B(X^B_{\cdot i},X^B_{\cdot j})$, see Eq. \ref{eq:kernel_functions}, and (ii) The number of significant eigenvectors computed for $G_A$ and $G_B$, denoted $d_A$ and $d_B$, respectively.

\paragraph{Self-tuning bandwidth}
For computing the weight matrices, we use the  self-tuning bandwidth from \cite{zelnik2004self} where the bandwidth for an RBF kernel 
 is given by $K(x_i,x_j) = \exp(\|x_i-x_j\|/\sigma_i\sigma_j) $. The local bandwidth $\sigma_i$ for each node is set to the distance to its $k$-th nearest neighbor, as suggested in \cite{zelnik2004self} and as is common in practice. The rule of thumb for choosing  $k$ is around  $\log(p)$ where $p$ is the number of features. 

\paragraph{Determining $d_A$ and $d_B$}

The notation $d_A,d_B$ represent the number of significant eigenvectors present in the random walk matrices of $P_A,P_B$ respectively. It is important to note that if $d_A,d_B$ are two small, the leading singular vectors of $P_AQ_B$ and $P_BQ_A$ will include elements of the shared latent space. The results will not change dramatically, however, if the choice of $d_A,d_B$ is higher than the optimum. For example, in the experiment in section 4.1 “Identifying newly connected features”, when we consider 
$d_A=d_B<4$, the shared latent space of features between 100-150 are also highlighted in 
 $V_A,V_B$ which is undesirable. However, we can increase the values of $d_A,d_B$ up to around 150, with very little impact on the results. For the MNIST data, we computed differential features between digits 4 and 9 with various values for $d_A=d_B=d$  and followed the procedure mentioned in the paper to compute the classification accuracy. These results are given in Table \ref{tab:hyperparameters}. We can see that for very small values of $d$
, the accuracy is lower since the classifier is partially trained on information about shared features. Here as well, there is a wide range of values (between 20-40) that yield similar results. 
\begin{table}[ht]
    \centering
    \begin{tabular}{|c|c|c|c|c|c|c|}
     \hline
    $d_A = d_B = d$     & 10 & 20 & 30 & 40 & 50 & 60 \\ \hline
     Test Accuracy    & 95.2\% & 96.5 \% & 96.5\% & 96.4\% & 94.7 \% & 88.5\%  \\
     \hline
    \end{tabular}
    \caption{Impact of the choice of hyperparameter $d$ on the classifier accuracy for pairs of MNIST digits.}
    \label{tab:hyperparameters}
\end{table}



\section{Proof of lemmas \ref{lem:expected_wb} and \ref{lem:concentation_v_gamma}}\label{app:proofs}

\subsection{Preliminaries}
\paragraph{The Davis-Kahan Theorem}
In our proof, we make a repeated use of the Davis-Kahan theorem. We apply both the classic theorem, and a useful variant derived in \cite{yu2015useful}. 
\begin{theorem}[\cite{yu2015useful}, Theorems 1 and 2]\label{thm:davis_kahan}
Let $W$ be a symmetric matrix with eigenvectors $v_1,\ldots,v_n$ and corresponding eigenvalues $\lambda_1 \geq \lambda_2,\ldots,\lambda_n$.
We denote by $\widetilde{W}$ a perturbation of $W$, with eigenvectors $\tilde v_1,\ldots,\tilde v_n$ and corresponding eigenvalues $\tilde  \lambda_1 \geq \tilde \lambda_2, \ldots,\tilde \lambda_n$.
Let $T_W,T_{\widetilde{W}}$ be the projection matrices onto a subspace spanned by the leading $d$ eigenvectors,
\[
T_W = \sum_{i = 1}^d v_iv_i^T \qquad T_{\widetilde{W}} = \sum_{i = 1}^d \tilde v_i\tilde v_i^T.
\]
In addition, we define $\delta$ via,
\[
\delta = \min_{j\leq d; i >d}|\tilde \lambda_i-\lambda_j|.
\]
Then,
\[
\|T_W-T_{\widetilde{W}}\| \leq \|W-\widetilde{W}\|/\delta.
\]
Alternatively, 
\[
\|T_W-T_{\widetilde{W}}\| \leq 2\sqrt{d}\|W-\widetilde{W}\|/(\lambda_{k}-\lambda_{k+1}).
\]
\end{theorem}
The first inequality is the classic Davis-Kahan theorem and the second is its variant derived in \cite{yu2015useful}. The importance of the variant is that it bounds the eigenvector perturbation as a function of the eigenvalues of original matrix $W$, with no dependency on the eigenvalues of the perturbed matrix, which are unknown in many cases. In addition, 
the bound on the projection matrices $T_W-T_{\widetilde{W}}$ can be replaced with a bound on the difference in norm between the subspace of eigenvectors, see for example Corollary 3 in \cite{yu2015useful}. For a single vector we have 
\begin{equation}
    \|\tilde v_i-v_i \| \leq \frac{2^{3/2}\|W-\widetilde{W}\|}{\min (\lambda_{i-1}-\lambda_{i},\lambda_{i}-\lambda_{i+1})}.
\end{equation}

\subsection{Concentration of weight matrix for stochastic block models}
Another useful result is the concentration, in spectral norm, of a weight matrix generated according to the stochastic block model. 
This result follows directly from Bernstein's inequality for sums of independent matrices with bounded norm. This derivation is clearly presented, for example, in \cite{vershynin2018high}
\begin{lemma}\label{lem:sbm_concentration}
Let $W\in \mathbb R^{l \times l}$ be a matrix generated by the stochastic block model as in Eq. \eqref{eq:sbm}. Then, 
\[
\|W - \mathbb E[W]\| = C\sqrt{l} \qquad \text{with probability} \qquad 1-\exp(-l).
\]
\end{lemma}

The rank of the expected weight matrix $\mathbb E[W]$ is equal to the number of communities in the model.   
Assume $d$ communities and let $T_W$ and $T_{\mathbb E[W]}$ denote the projection matrices onto the leading $d$ eigenvectors of $W$ and $\mathbb E[W]$ respectively. Let $\lambda_d$ denote the $d$-th eigenvalue of $\mathbb E[W]$.
Combining Lemma \ref{lem:sbm_concentration} and  Theorem \ref{thm:davis_kahan} yields the following  perturbation bound,
\begin{equation}\label{eq:sbm_general}
\|T_W-T_{\mathbb E[W]}\| \leq \frac{C\sqrt{dl}}{\lambda_d}.
\end{equation}
where $C$ is a constant that does not depend on the parameters of the model.

\subsection{Auxiliary lemmas}
For the lemmas in this subsection we have the following notation.
Let $W_A,W_B \in \mathbb R^{(2l+s) \times (2l+s)}$ be random weight matrices obtained via the stochastic block model as described in Section \ref{sec:double_sbm}.
Let $Q_{W_A},Q_{\WB}$ denote two  projection matrices onto the complementary subspace of the leading eigenvectors of $W_A$ and $\WB$ respectively.

\begin{lemma}\label{lem:double_projection_diff}
We have the following bound on the numerator of Eq. \eqref{eq:dk_application}.
\begin{align*}
\|Q_{W_A} W_B Q_{W_A} - Q_{\mathbb E[W_B]}\mathbb E[W_B]Q_{\mathbb E[W_B]}\| \leq 
  C_1\sqrt{l} + C_2 \sqrt{\frac{s}{l}}\lambda_3
\end{align*}
\end{lemma}
\begin{proof}
We denote by $\mathcal E = Q_{W_A}-Q_{\WB}$. Applying the triangle inequality and the Cauchy-Schwarts inequality, 
\begin{align}\label{eq:bound_qwq}
    &\|Q_{W_A} W_B Q_{W_A} - Q_{\mathbb E[W_B]}\mathbb E[W_B]Q_{\mathbb E[W_B]}\| 
    \notag \\
    &=\|Q_{W_A}  W_B (Q_{\mathbb E[W_B]} + \mathcal E) - (Q_{W_A} -\mathcal E)\mathbb E[W_B]Q_{\mathbb E[W_B]}\|
    \\
    &\leq \| Q_{W_A}(W_B-\mathbb E[W_B])Q_{\mathbb E[W_B]}\|
    +\|\mathcal E \WB Q_{\mathbb E[W_B]}\| + \|Q_{W_A} W_B \mathcal E\|.\notag 
    \\
    & \leq 
    \| W_B-\mathbb E[W_B]\|
    +\|\mathcal E \| \|\WB Q_{\mathbb E[W_B]}\| + \|Q_{W_A} W_B \|\|\mathcal E\|.\notag
\end{align}
The term $\|\WB Q_{\mathbb E[W_B]}\|$ is equal by definition to the third eigenvalue of $\WB$.
We combine lemmas \ref{lem:sbm_concentration}, \ref{lem:bound_e} and \ref{lem:bound_QW} to get,
\[
\|Q_{W_A} W_B Q_{W_A} - Q_{\mathbb E[W_B]}\mathbb E[W_B]Q_{\mathbb E[W_B]}\| \leq 
C\sqrt{l} + \sqrt{\frac{s}{l}}\lambda_3 +C_2 \sqrt{s}.
\]
\end{proof}

\begin{lemma}\label{lem:bound_e}
The error of $\|\mathcal E\| = \|Q_{W_A}-Q_{\WB}\|$ is bounded by,
\[
\|\mathcal E\| \leq C\sqrt{\frac{s}{l}}. 
\]
\end{lemma}
\begin{proof}
The triangle inequality yields,
\begin{align*}
\|Q_{W_A}-Q_{\WB}\| 
\leq \|Q_{W_A}-Q_{\WA}\| + \|Q_{\WA}-Q_{\WB}\|
\end{align*}
For the first term, we use the results for the standard stochastic block model of sizes $l$ and $l+s$.
The second (and smallest non-zero) eigenvalue of $\WA$ is larger than $(p-q)\frac{l}{2}$
\[
\|Q_{W_A}-Q_{\WA}\| \leq \frac{C\sqrt{l}}{(p-q)l} = \frac{C}{\sqrt{l}(p-q)}.
\]
The second term bounds the difference in projection matrices of the subspace spanned by the leading two eigenvectors of $\WA$ and $\WB$. In the proof of lemma \ref{lem:expected_wb} we show that the difference between the leading eigenvectors of the two matrices is bounded by $\sqrt{\frac{s}{l}}$ and that the second eigenvector is identical. It follows that,
\[
\|Q_{W_A}-Q_{\WB}\| \leq \frac{C}{\sqrt{l}(p-q)} + \sqrt{\frac{s}{l}}. 
\]
\end{proof}
\begin{lemma}\label{lem:bound_QW}
The value of $\|Q_{W_A} W_B\|$ is bounded by,
\[
\|Q_{W_A} W_B\| \leq C \sqrt{l}.
\]
\end{lemma}
\begin{proof}
We use the triangle inequality and Cauchy Schwartz to split the term into the following,
\begin{align}
\|Q_{W_A} W_B\| &\leq \|Q_{W_A} W_A\| + \|Q_{W_A} (W_A-W_B)\| \leq 
\|Q_{W_A} W_A\| + \|Q_{W_A} (\WA-\WB)\| \notag
\\
&+ \|Q_{W_A} (W_B-\WB)\| + \|Q_{W_A} (W_A-\WA)\|
\\
& 
\leq \|Q_{W_A} W_A\| + \|W_B-\WB\| + \|W_A-\WA\|  + \|Q_{WA}(\WA-\WB)\|\notag \\
&\leq 
\|Q_{W_A} W_A\| + \|W_B-\WB\| + \|W_A-\WA\|+\|Q_{\WA}(\WA-\WB)\| \notag \\
&+\|\WA-\WB\| \|Q_{\WA}-Q_{W_A}\|  \notag 
\end{align}
By lemma \ref{lem:sbm_concentration} (concentration of the norm of random matrix) The terms $\|W_B-\WB\|$ and $\|W_A-\WA\|$ are bounded by $C \sqrt{l}$. The term $\|Q_{W_A}W_A\|$ is equal by definition to the third eigenvalue of $W_A$. Recall that the third eigenvalue of $\WA$ is equal to zero. Thus, by weyl's inequality, the third eigenvalue of $W_A$ is bounded by
\[
\lambda_3(W_A) \leq  \|W_A-\WA\| \leq C \sqrt{l}.
\]
To bound the term $\|Q_{WA}(\WA-\WB)\|$ note that the matrix $\WA-\WB$ is deterministic, and equal to
\[
\WA-\WB = (p-q)e_\beta e_\gamma^T
\]
with a norm of $(p-q)\sqrt{sl}$. The operator $Q_{W_A}$ removes the average from the vector $v_\beta$ over the $l+s$ last elements. Thus, the norm of  $\|Q_{\WA}e_\beta\|$ is bounded by $s/l$, which implies that 
\[
\|Q_{\WA}(\WA-\WB)\| \leq (p-q) s^{3/2}/l.
\]
Finally, the last term is bounded by $C \frac{\sqrt{sl}}{(p-q)\sqrt{l}} = C\frac{\sqrt{s}}{p-q}$, Which is dominated by the first four terms.  
Summing up the different terms we have,
\[
\|Q_{W_A}W_B\| \leq C\sqrt{l} + (p-q) s^{3/2}/l.
\]
\end{proof}

\subsection{Proof of lemma \ref{lem:expected_wb}}

\begin{proof}

We define $e_\alpha$, $e_\beta$ and $e_\gamma$ as the binary block indicator vectors where $(e_\alpha)_i=1$ if $i \in \alpha$, $(e_\beta)_i=1$ if $i \in \beta$ and $(e_\gamma)_i=1$ if $i \in \gamma$. 

Let $E \in  \mathbb R^{(2l+s) \times 3}$ be a concatenation of $e_\alpha,e_\beta$ and $e_\gamma$.
We denote the pairwise block confusion matrix $\Theta \in [0,1]^{3 \times 3}$ given by,
\begin{equation}\label{eq:sbm_matrix}
\Theta_{ij} = 
\begin{cases}
p & i = j\\
q & i \neq j.
\end{cases}
\end{equation}
The expected weight matrix of the stochastic block model is equal to,
\[
\WB = E \Theta E^T.
\]

We denote by $\Delta\in R^{3 \times 3}$ a diagonal matrix with,
\[
\Delta_{11} = \sqrt{l} \quad \Delta_{22} = \sqrt{l} \quad \Delta_{33} = \sqrt{s}.
\]
The expected weight matrix of the stochastic block model is equal to 
\[
\WB = E\Theta E^T = (E \Delta^{-1}) (\Delta \Theta \Delta) (\Delta^{-1}E^T).
\]
The matrix $E \Delta^{-1}$ is orthonormal. The eigenvalues of $\WB$ are thus equal to the eigenvalues of $\Delta \Theta \Delta$ and the corresponding eigenvectors are equal to $E \Delta^{-1}$ multiplied by the eigenvectors of $\Delta\Theta\Delta$.
Consider the matrix $Z$ given by,
\[
Z_{ij} = \begin{cases}
(\Delta\Theta\Delta)_{ij} & i,j <3 \;\; \text{or} \;\; i=j=3\\
0 & o.w.
\end{cases}
\]
The eigenvectors of $Z$ are equal to $u_1 = [1,1,0]$  and $u_2 = [1,-1,0]$ with corresponding eigenvalues $(p+q)l$ and $(p-q)l$ respectively. 
We denote by $\tilde u_1,\tilde u_2,\tilde u_3$ the eigenvectors of $\Delta \Theta \Delta$. 
A direct computation shows that $\tilde u_2 = u_2$ with the same eigenvalue.
Applying Theorem \ref{thm:davis_kahan} yields, 
\begin{align}\label{eq:bound_u1_error}
&\|u_1 - \tilde u_1\| \leq \frac{2\sqrt{2}\|Z - \Delta\Theta\Delta\|}{2ql} = 
\frac{4\sqrt{2}q\sqrt{ls}}{2ql} = \sqrt{\frac{8s}{l}}.
\end{align}
Both $u_1,u_3$ and $\tilde u_1$, $\tilde u_3$ are orthogonal to $u_2=\tilde u_2$ and thus span the same 2D subspace. This implies
\[
\|\tilde u_3- u_3\| = \|u_1-\tilde u_1\| \leq  \sqrt{\frac{8s}{l}}.
\]
It follows directly that $v_\gamma$, the the third eigenvector of $\WB$
satisfies 
\[
\Big\|v_\gamma- \frac{1}{\sqrt{s}}e_\gamma\Big\| \leq \sqrt{\frac{8s}{l}}. 
\]
Finally, we derive a lower bound on the third eigenvalue of $\WB$. Recall that the eigenvalues of $\WB$ are equal to those of $\Delta \Theta \Delta$. For the latter we apply the inequality,
\[
\lambda_{\text{min}}(\Delta \Theta \Delta) \geq \lambda_{\text{min}}(\Delta) \lambda_{\text{min}}(\Theta) \lambda_{\text{min}}(\Delta),  
\]
where $\lambda_{\text{min}}$ denotes the smallest eigenvalue. The matrix $\Delta$ is diagonal with values $\sqrt{l},\sqrt{l},\sqrt{s}$ and hence $\lambda_{\text{min}}(\Delta) = \sqrt{s}$. By direct computation $ \lambda_{\text{min}}(\Theta)\geq (p-q)$. Hence,
\[
\lambda_3(\WB) = \lambda_{\text{min}}(\Delta \Theta \Delta) \geq s(p-q).
\]

\end{proof}

\subsection{Proof of Lemma \ref{lem:concentation_v_gamma}} 
In this lemma we bound the difference between the vector $\tilde v_\gamma$ and the  corresponding eigenvector of $\WB$, denoted by $v_\gamma$. 
Recall the definition of $Q_{\WB}$ as the projection matrix onto the complementary subspace of the two leading eigenvectors of $\WB$. 
By definition, $v_\gamma$ is the leading eigenvector of $Q_{\WB} \WB Q_{\WB}$. Thus, our goal is to bound the leading eigenvector of two matrices:
\[
Q_{\WB}\WB Q_{\WB} \quad \text{and} \quad Q_{\WA}W_BQ_{\WA}.
\]
To that end, we apply the Davis-Kahan theorem. Let $\lambda_\gamma$ be the third eigenvector of $\WB$. The theorem gives the following bound. 
\begin{align}\label{eq:dk_application}
\|v_\gamma-\tilde v_\gamma\|  
\leq2^{3/2}\frac{\|Q_{\WB}\WB Q_{\WB}  -Q_{\WA}W_BQ_{\WA}\|}{\lambda_3}.
\end{align}

Applying lemma \ref{lem:double_projection_diff} we get:
\[
\|v_\gamma - \tilde v_\gamma\|\leq C\frac{\sqrt{l}}{\lambda_3}+C_2\sqrt{\frac{s}{l}}
\]
For the first term, we apply the lower bound $\lambda_3 \geq (p-q)s$. This yields,
\[
\|v - \tilde v\| \leq C_1\frac{\sqrt{l}}{s} +  C_2\sqrt{\frac{s}{l}}
\]
In Section \ref{sec:validation_theory} we provide a numerical validation of Lemma \ref{lem:double_projection_diff} and the bound on $\lambda_3$.


\section{Alternative justification for our approach}
A different variation of the graph-cut criterion in \ref{eq:ratio_cut} is the Normalized cut. Let $d_i$ denote the degree of node $i$, and let $\textrm{vol}(\alpha) = \sum_i d_i$ be the sum of degrees of a subset of nodes $\alpha$. The normalized cut is equal to,
\[
\textrm{Ncut}(\alpha,\beta) = \sum_{i \in \alpha,j \in \beta} W(i,j)\Big( \frac{1}{\textrm{vol}(\alpha)} + \frac{1}{\textrm{vol}(\beta)\ }\Big).
\]
Similarly to \ref{eq:ratio_cut_vector}, we can define an indicator vector by,
\begin{equation}\label{eq:indicator_ncut}
f_i = 
\begin{cases}
\sqrt{\frac{\textrm{vol}(\beta)}{\textrm{vol}{\alpha}}} & i \in \alpha \\
\sqrt{\frac{\textrm{vol}(\alpha)}{\textrm{vol}{\beta}}} & i \in \beta.
\end{cases}
\end{equation}
The normalized cut can be expressed as,
\[
\textrm{Ncut}(\alpha,\beta) = f^T L f.
\]
Here, the indicator vector is orthogonal to $D \bm 1$, where $D$ is a diagonal matrix containing the degrees $\{d_i\}$ in its diagonal and $\bm 1$ is the all ones vector. In addition, though the norm of $f$ depends on $\alpha$, the term $f^T D f$ is a constant. Relaxing the requirement in Eq. \ref{eq:indicator_ncut}, yields the following optimization problem for the normalized cut,
\[
\min_f f^T L f \qquad \textrm{subject to} \qquad f^T D f = 1.
\]
In our work, we consider two graph Laplacians $L_A,L_B$. For simplicity we assume that the degree matrix of both graphs is identical. Thus, the indicator vector in Eq. \ref{eq:indicator_ncut} is the same for both graphs and for any partition. 
The second graph $L_B$ yields a second constraint on the optimization problem.
\begin{equation}\label{eq:optimization_ncut}
\min f^T L_A f  \quad \textrm{subject to} \quad f^T D f = 1, \quad f^T L_B f = \gamma. 
\end{equation}
The solution to Eq. \ref{eq:optimization_ncut} satisfies the generalized eigenvector problem 
\[
L_A f = \lambda_1(D+\lambda_2L_B)f.
\]
Multiplying both sides by $D^{-1}$ yields,
\begin{equation}\label{eq:reg_inverse}
(I + \lambda_2P_b)^{-1} P_A f = \lambda_1 f,
\end{equation}
where $P_A,P_B$ are the random walk Laplacian matrices. Note that the term $(I + \lambda_2 D^{-1}L_b)^{-1}$ is a \textrmit{regularized inverse} matrix of $P_B$, often used to ovoid an arbitrary increase of noise when computing the inverse of ill condition matrices. Thus, the expression in \ref{eq:reg_inverse} is very similar in nature to the expression in Eq. \ref{eq:generalized_eigenvector}.

\section{DiSC extension to multiple datasets}
\label{sec:disc_multiple}

The extension of our approach to multiple datasets (i.e. more than two) is straightforward. We consider $M$ datasets $X^m \in \mathbb R^{n_m \times p}$ for $m=1,..,M$. We compute random-walk transition matrices on the graphs by $P_m=D_m^{-1}W_m$. Let $U_m \in \mathbb R^{p \times d_m}$ be a matrix containing the $d_m$ leading right eigenvectors of $P_m$, and let $\hat{Q}_m$ be a projection matrix onto the the complementary subspace of  $\cup_{k\neq m}U_k$, given by
\begin{equation}\label{eq:projection_m}
\hat{Q}_m = I - \hat{U}_m (\hat{U}_m^T \hat{U}_m)^{-1} \hat{U}_m^T.
\end{equation}
where $\hat{U}_m = [U_1,...,U_{m-1},U_{m+1},...,U_M]$.
We then compute the differential vectors of $G_m$, denoted by $(V_m)_i$ and the corresponding significance level $(\sigma_m)_i$ using
\begin{equation}\label{eq:differentiating_vectors_m}
     (V_m)_i = 
        \argmax_{\textrmrm{dim} E=i} \min_{v\in S(E)} ||P_m\hat{Q}_mv||_2.
\end{equation}
\begin{equation}\label{eq:diff_vectors_significance_m}
     (\sigma_m)_i = ||P_m\hat{Q}_m(V_m)_i||_2.
\end{equation}
where $E$ is a subspace in $\mathbb R^{p}$ and $S(E)$ denotes the unit Euclidean sphere in $E$. An algorithm implementing our approach for multiple datasets is given in Alg.~\ref{alg_multiple}.

\begin{algorithm}[ht]
\caption{DiSC - Extension to multiple datasets}
\begin{flushleft}
\textrmbf{Input:} 
\begin{tabular}[t]{ll}
Datasets $X^m \in \mathbb R^{n_m \times p}$ for $k=1:M$ \\
Kernel functions $\{K_m(\cdot,\cdot)\}_{k=1:M}$  \\
Hyper parameters $\{d_m\}_{k=1:M}$ \\
\end{tabular}
\\
\textrmbf{Output:} Subsets of differentiating features $\{V_m\}_{m=1:M}$
\end{flushleft}
\begin{algorithmic}[1]
\State Compute graphs $G_m$ on the columns of $X^m$ with weights given by \ref{eq:kernel_functions}.
\State For all the graphs, compute the random walk matrix,
\[
P_m=D_m^{-1}W_m \qquad \textrm{for $m=1,...,M$}
\]
\State Calculate $\{U_m\}_{m=1:M}$, the leading right eigenvectors of $\{P_m\}_{m=1:M}$. 
\item Compute the projection matrices $\{\hat{Q}_m\}_{m=1:M}$ via Eq. \ref{eq:projection_m}. 
\State Compute differential vectors $\{V_m\}_{m=1:M}$ via Eq. \ref{eq:differentiating_vectors_m}.
\State Compute significance levels  $\{\sigma_m\}_{m=1:M}$ via Eq. \ref{eq:diff_vectors_significance_m}.
\State optional: Perform k-means over the rows of $V_A$, and $V_B$.
\end{algorithmic}
\label{alg_multiple}
\end{algorithm}

To demonstrate this extension, in the following example we generalize our approach to reveal differential features given three datasets $X_A, X_B$ and $X_C$. Some of the features connections are specific to each dataset and the remaining connections may be present in more than one dataset. 
Each dataset is consists of $p=400$ features and $n_A=n_B=n_C=10,000$, with feature correlations as shown in  Fig.~\ref{fig:toy4}(a). Dataset specific groups of correlated features are highlighted. 
That is, feature connections between 301-350, 201-250 and 351-400 are specific to $X^A$, $X^B$ and $X^C$, respectively.
Note that each dataset has two groups of correlated features that are dataset specific and the aim is to identify these two groups of features individually. We apply DiSC with $d_A=d_B=d_C=20$. Differential features are shown in Fig.~\ref{fig:toy4}(c).
The significance level of the first two differential vectors is almost similar for all the datasets and then significance level drops, as shown in Fig.~\ref{fig:toy4}(b). This indicates that the first two differential vectors capture the major differences. These vectors, shown in Fig.~\ref{fig:toy4}(b), clearly represent the dataset specific connections and each differential vector identifies a group of connected features, unlike diffusion maps in Fig.~\ref{fig:toy4}(d).

 \begin{figure}[th]
     \centering
     \includegraphics[width=0.9\textwidth]{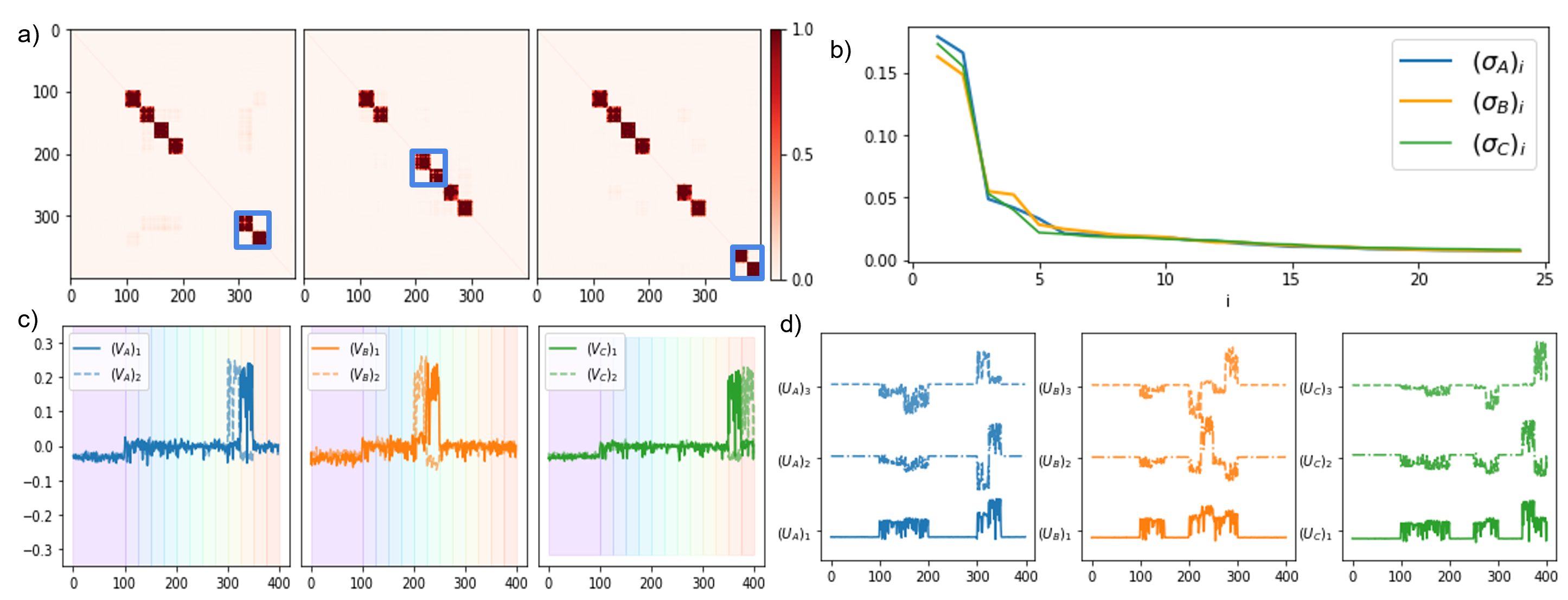}
    \caption{Multiple datasets. a) Correlation matrices. b) Significance levels. c)~ Differential features. d) Diffusion maps.}
     \label{fig:toy4}
 \end{figure}

\section{Additional experimental results}\label{appendix:experiments}
\paragraph{Identifying subsets of connected features in both datasets}
We generate a toy problem with two datasets $X^A$ and $X^B$ with $p=200$ features and $n_A=n_B=10,000$, whose features are correlated as in Fig.~\ref{fig:toy3}(a).
There is a subset of correlated features in $X^A$ that are divided into two groups in $X^B$ and vice versa. $X^A$ has three groups of correlated features with feature indices [1-75],[76-100] and [101-200]. $X^B$ has another three groups of correlated features with feature indices [1-100],[101-125],[126-200]. 
Thus, feature indices [1-100] which have strong connectivity in $X^B$ are divided into [1-75] and [76-100] groups in $X^A$. Similarly, feature indices [101-200] which have strong connectivity in $X^A$ are divided into [101-125] and [126-200] groups in $X^B$. The goal is to identify these sub-divided groups i.e., feature indices [1-75],[76-100] in $X^A$ and  [101-125],[126-200] in $X^B$ as the differential features. Significant differential vectors and the corresponding significance values are shown in Fig.~\ref{fig:toy3}(c) and (b) respectively. The significance level of the differential vector in both datasets is roughly similar and the significance level drops after the first vector. Therefore, the first differential vector is the most significant one. Fig.~\ref{fig:toy3}(c) shows that the differential vectors are clustering the two subsets of differential features.

\begin{figure}
    \centering
    \includegraphics[width=\textwidth]{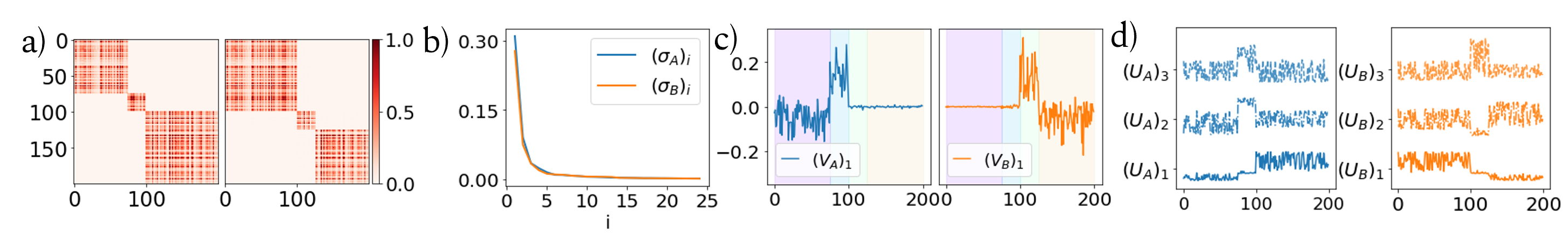}
    \caption{a) Correlation matrices. b) Significance levels c)~ Differential features d) Eigenvectors of diffusion maps.}
    \label{fig:toy3}
\end{figure}

\subsection{Hyperspectral imaging}
We provide two additional results of DiSC on the hyperspectral imaging dataset.
To illustrate the feature grouping capability of DiSC, we choose $X^A=X^{oct}$ and $X^B=X^{oct-c}$, and set the hyperparameters $d_A=d_B=20$. 
The top four differential features of $X^A$ and $X^B$ are shown in Fig.~\ref{fig:hyper_oct_octc} in the top row and bottom row respectively. $(V_B)_3$ captures the tarp and $(V_B)_4$ captures the slight variation in the trees in the background. $(V_B)_1$ and $(V_B)_2$ captures the change in the grass. Clearly, there is a grouping effect and each group of differential features are captured individually. Also, since the entire information in $X^A$ is present in $X^B$,  $V_A$ do not capture any differences.

\begin{figure}
    \centering
    \includegraphics[scale=0.58]{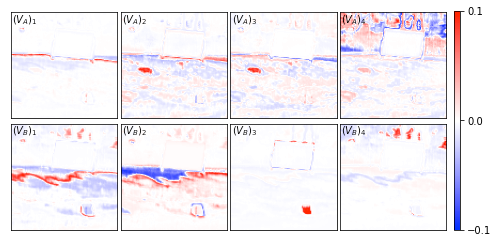}
    \caption{The differential vectors for $X^A=X^{oct}$ (top row) and $X^B=X^{oct-c}$ (bottom row).}
    \label{fig:hyper_oct_octc}
\end{figure}

To evaluate the robustness of our approach to noise, we added Gaussian noise with $\mu=0$ and $\sigma=0.01$ to all datasets. We then compute the differential features using DiSC and distance measure using DM-changing data. We cluster the features using these with k-means clustering and $k$=3. Fig.~\ref{fig:hyper_noise} shows the clusters for four pairs of datasets using DiSC  (top row) and DM-changing data (bottom row). DiSC is more robust to noise than DM-changing data. In most months it groups the pixels belonging to the tarp as a practically separate cluster of differential features, as opposed to DM-changing data which groups the tarp with other features in the background. Thus, DiSC is more invariant to imaging conditions and highlights the object that has changed in the scene.

\begin{figure}
    \centering
    \includegraphics[scale=0.42]{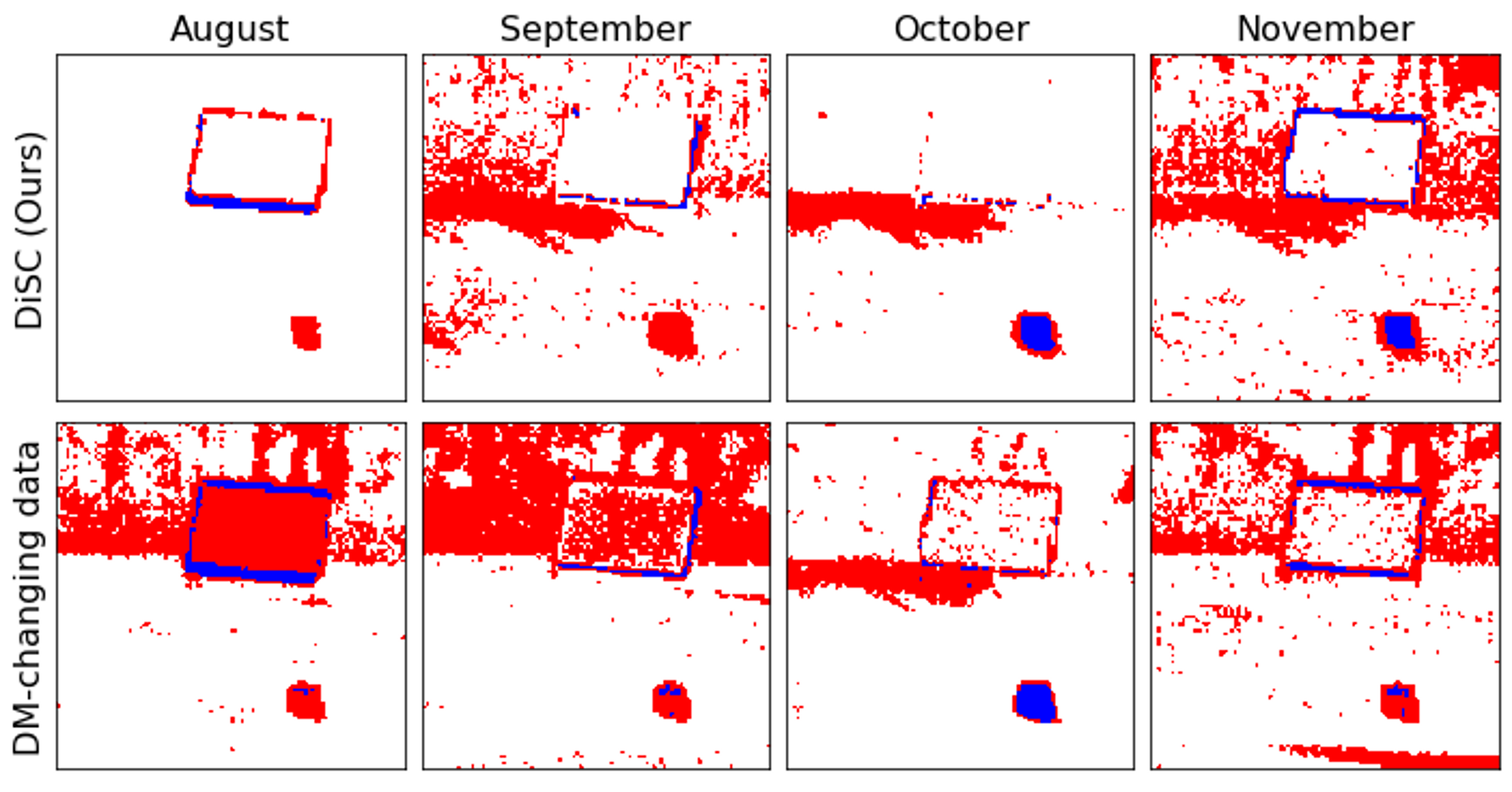}
    \caption{Noisy hyperspectral data. Three groups of features (white, red and blue) formed from DiSC (top row) and Diffusion Maps for changing data (bottom row).}
    \label{fig:hyper_noise}
\end{figure}

\subsection{MNIST pairwise classification} \label{sec:mnist_appendix}
The MNIST dataset \citep{lecun-mnisthandwrittendigit-2010} consists of images of hand written digits from 0 to 9, each of dimension $28\times28$ pixels. It has 60k training samples and 10k testing samples. We extract the differential features between each pair of digits with $d_A=d_B=20$. The top 3 differential features from the two digits are further concatenated columnwise to form a matrix of dimension $784 \times 6$. 
For each pair of digits, we group the features into three clusters using the k-means algorithm applied to this matrix of concatenated differential features. 
We compare our approach with Naive Elastic Net (EN), an Elastic Net variation for classification (EN-logistics) and Diffusion Maps applied to the features of each dataset. 
Each of these methods yields a vector(s) representing feature importance, which is then used for feature grouping  using k-means with $k=3$. As an additional baseline, we group the features based on the entire data. To have a quantitative metric to measure the performance of these methods, we first compute the average feature values for each cluster of features and use these to train a logistic regression classifier between pairs of digits, and measure the classification accuracy on the test samples. 

For each digit, Fig.~\ref{fig:mnist_all} shows the average of its pairwise accuracy with the other digits, for different methods.
Our method consistently performs better than the other methods, with relatively less variability in accuracy across different pairs of digits. The baseline model and diffusion maps have a huge variability, which may be because these are not designed to extract differences between the digits. In most cases, EN-logistic performs better than diffusion maps because EN-logistic assigns feature importance based on the classification task. However, it only has a single set of coefficients for features, whereas our method has a number of differential features with associated significance level. As illustrated in Fig.~\ref{fig:mnist_all} for digits $4$ and $9$, diffusion maps captures the overall structures whereas DiSC explicitly capture their differences.

\begin{figure}
    \centering
    \includegraphics[width=\linewidth]{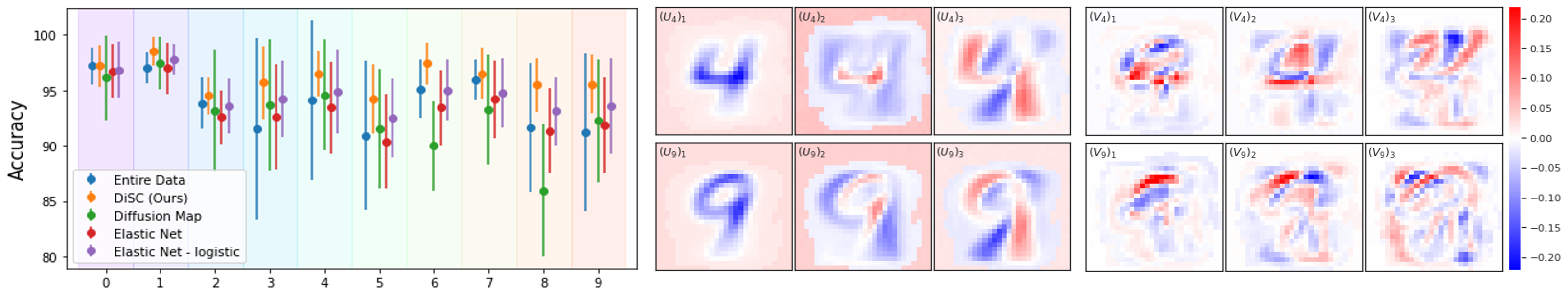}
    \caption{Results for MNIST data. The figure shows the accuracy of various methods for classifying a pair of digits.
    For each digit (x-axis) we show the average accuracy over all pairs.
    }
    \label{fig:mnist_all}
\end{figure}

\subsection{single cell RNA sequencing} 

Using the \textrmit{splatter} simulator \citep{zappia2017splatter} we generated a simulated dataset containing the RNA expression level of $p=500$ genes, as measured in $1000$ cells, which belong to two different cell types. Out of the $500$ genes, only approximately $100$ genes are \textrmit{differential}, and thus have a different expression level for the two cell types. 

We visualize the cells in 2D using t-SNE \citep{van2008visualizing}. 
Figure \ref{fig:scrna_sim} 
colors the cells by \textrmit{meta features}, that are equal to weighted sum over the gene expression profile of each cell. In the upper four panels, the weights are computed by the two leading diffusion vectors of the graphs $G_A$ (left two panels) and $G_B$ (right two panels). 
In the bottom four panels, the weights are equal to the differential vectors $V_A$ (left two panels) and $V_B$ (right two panels). Clearly, the differential vectors highlight genes that are more relevant for differentiating between the two groups.

\begin{figure}[ht]
    \centering
    \includegraphics[scale=0.14]{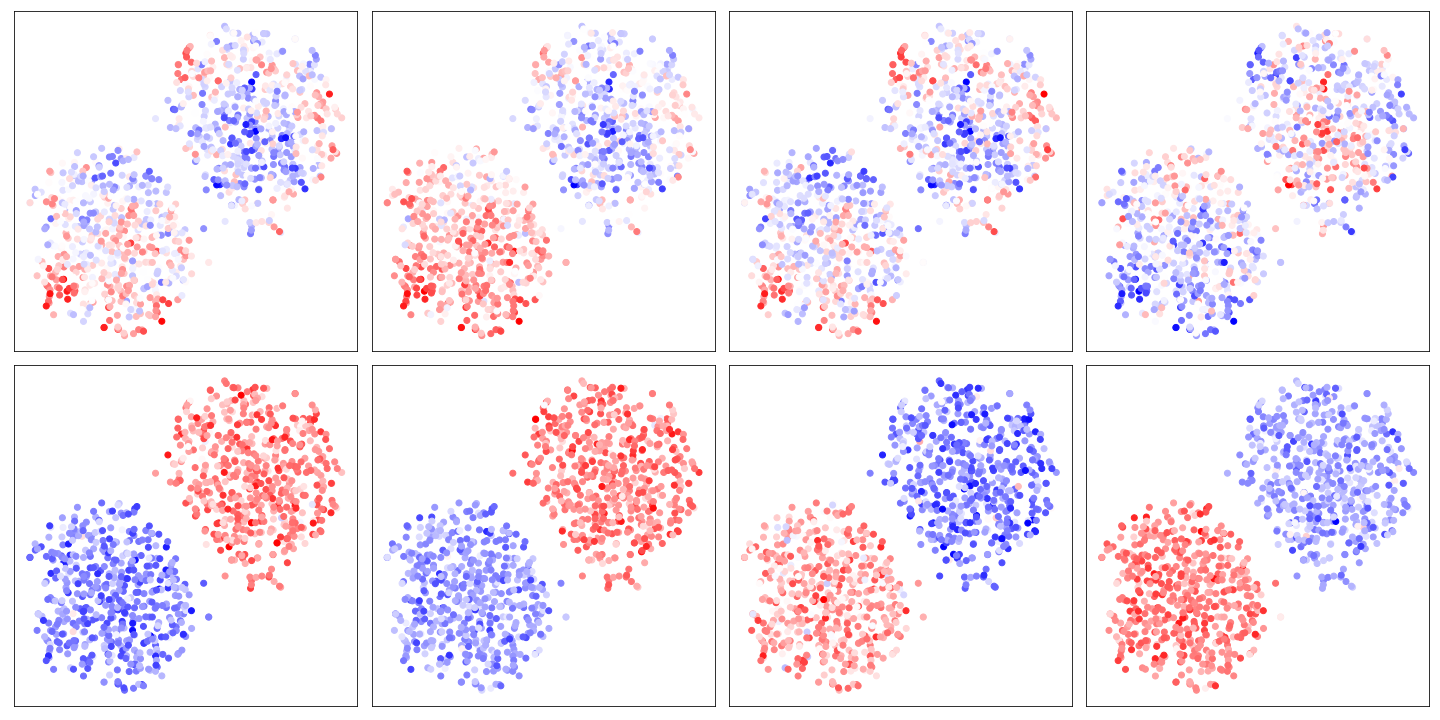}
    \caption{2D t-SNE plots of simulated scRNA-seq generated by splatter \cite{zappia2017splatter}. The two clusters represent two cell types. 
    In all panels cells are colored by \textrmit{meta-features} - a weighted sum on the gene expression, corresponding to diffusion maps on each of the feature graphs separately (top row), or DiSC vectors (bottom row).}
    \label{fig:scrna_sim}
\end{figure}

\subsection{Differential feature extraction in partially correlated conditions} 

\begin{figure}[ht]
   \centering
    \includegraphics[width=0.68\textwidth]{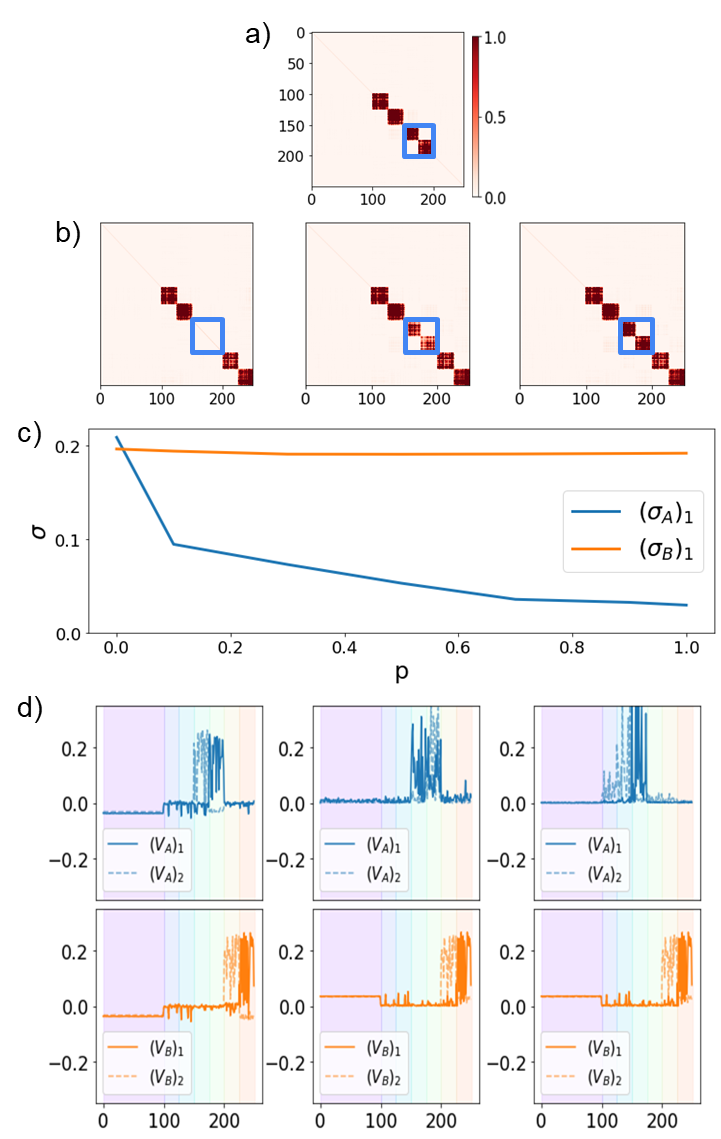}
  \caption{a) feature correlation matrix for dataset $A$. b) feature correlation matrix for dataset $B$ where p=0 (left), p=0.5 (middle) and p=1 (right). c) Significance of the first differential vector of $X^A$ (blue) and $X^B$ (orange). d) First two differential vectors of $X^A$ (top row) and $X^B$ (bottom row) for p=0 (left), p=0.5 (middle) and p=1 (right).}
  \label{fig:non_ideal}
\end{figure} 

We empirically analyze the performance of detecting differential features in a non-ideal case, where, say, correlated components in one dataset are partially correlated in other dataset. In section~\ref{sec:toy_problem}, we compared two ideal settings - features that are highly correlated in one dataset are completely independent in another. Here, we added an extension to this experiment to test non ideal cases. To that end, we introduce a parameter $p$ that determines the correlation level. Specifically, we set the eigenvalues of the covariance to decrease exponentially such that the $i^{th}$ eigenvalue $\lambda_i$ is proportional to $exp(-i*p)$. Thus, if $p$ is close to 0 - the decrease is very small, which results in a covariance matrix close to identity and thus no correlation. If p is high, then the decrease is fast which results in high correlation. In the experiment, we kept the covariance of features 151-200 in $X^A$ fixed with $p = 1$. For these features of $X^B$ we changed the values of $p$ between 0 and 1 and computed the differential vectors and  significance level for each of the $p$ values. As expected the significance level of the differential vectors of $X^A$ decreases as p gets close to 1. For intermediate levels, the significance level is still high.  The results are shown in Figure~\ref{fig:non_ideal}.

\subsection{Iterative baselines on simulated dataset} 
We applied the best subset selection approach to the simulated example in section~\ref{sec:toy_problem}, Identifying newly connected features. Here, our proposed methodology was able to identify features 151-200 and 201-250 as the differential features in $X^A$ and $X^B$ respectively. We applied the iterative feature selection approach to select the top 100 features that best differentiate $X^A$ and $X^B$. For our best subset selection, as a criterion, we used the accuracy of a nearest neighbor (NN) and logistic regression classifiers, that take as input the selected features. The accuracy was very poor compared to our group feature selection method, as shown in figure \ref{fig:iterative_baseline}. The figure on the left (right) indicates the selected features while using NN classifier (logistic regression classifier). The selected features are indicated by value 1. Clearly, these methods do not highlight the groundtruth 151-250 features as the features that best distinguish $X^A$ and $X^B$. 

\begin{figure}[ht]
    \centering
    \includegraphics[scale=0.7]{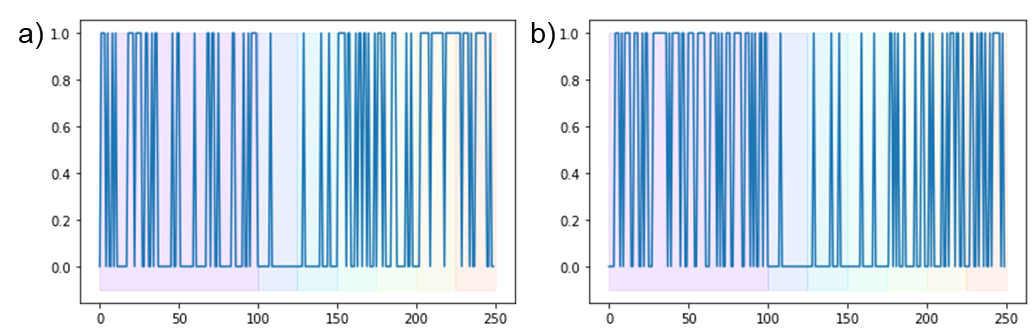}
    \caption{Iterative subset selection approach with NN (a) and logistic regression (b) classification accuracy criterion.}
    \label{fig:iterative_baseline}
\end{figure}

\subsection{Simulations for validating Lemma  \ref{lem:concentation_v_gamma}}\label{sec:validation_theory}
We ran several simulations to validate the theoretical results obtained in the proof of Lemma \ref{lem:concentation_v_gamma}. 
In our simulations, the block size $l$ varied from $500$ to $2000$. We set $s = l^{\alpha}$ for various values of $\alpha$ between $0.6$ and $0.9$. Our goal is to obtain numerical approximation to the rate of convergence of several factors in our proof, as a function of the block size $l$. 
Specifically, we estimate the increase rate of the numerator and denominator of Eq. \eqref{eq:dk_application}. The denominator is equal to $\lambda_3$, the third eigenvalue of $W_B$. By Lemma \ref{lem:expected_wb}, $\lambda_3$ is proportional to $s=l^\alpha$. Figure \ref{fig:theory_denominator} draws $\lambda_3$ as a function of $l$ on a log scale. The theoretical vs. numerical slope value is written over each panel. The numerical value matches the theoretical value almost perfectly. 

Next, we repeat the same experiment to compare the numerical and theoretical increase for the numerator of Eq. \eqref{eq:dk_application}.
Since $\lambda_3$ is proportional to $s$, The theoretical increase of the numerator is $O(l^{0.5} + l^{(3\alpha/2-0.5)})$. Figure \ref{fig:theory_numerator} shows the numerator vs. $l$ 
on a log-log scale. Similarly to Figure \ref{fig:theory_denominator}, for each panel, we write the numerical value of the slope, vs. the theoretical one. As expected, for low values of $\alpha$, the dominant term is $O(l^{0.5})$. For larger values of $\alpha$, the increase is slightly slower than expected from theory, implying that the bound on the numerator can be improved. 

\begin{figure}[ht]
    \centering
    \includegraphics[width = 0.85\textwidth]{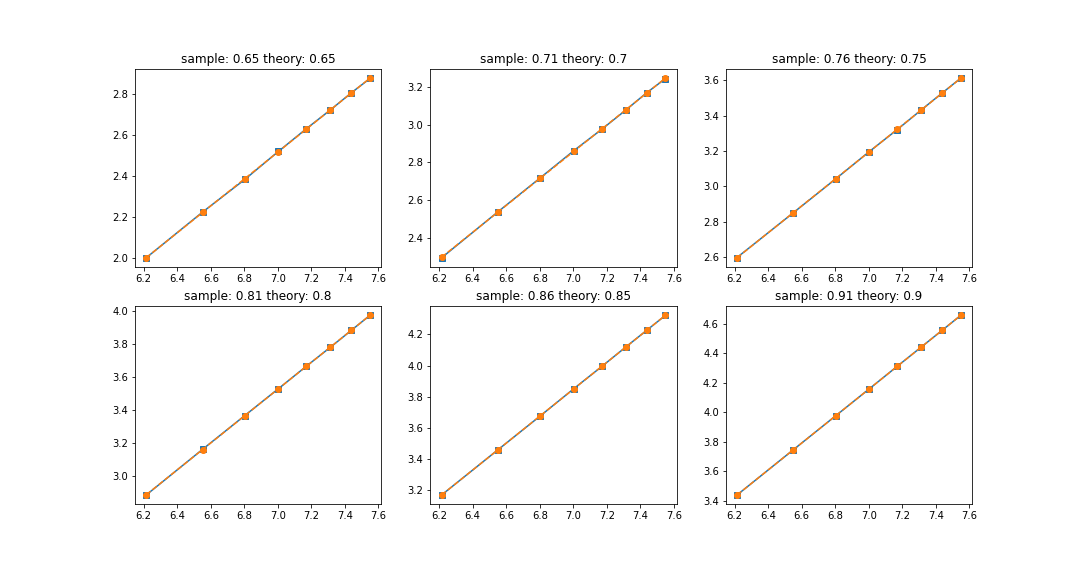}
    \caption{Caption: convergence of the denominator of Eq. \eqref{eq:dk_application}}
    \label{fig:theory_denominator}
\end{figure}
\begin{figure}[ht]
    
    \centering
    \includegraphics[width = 0.85\textwidth]{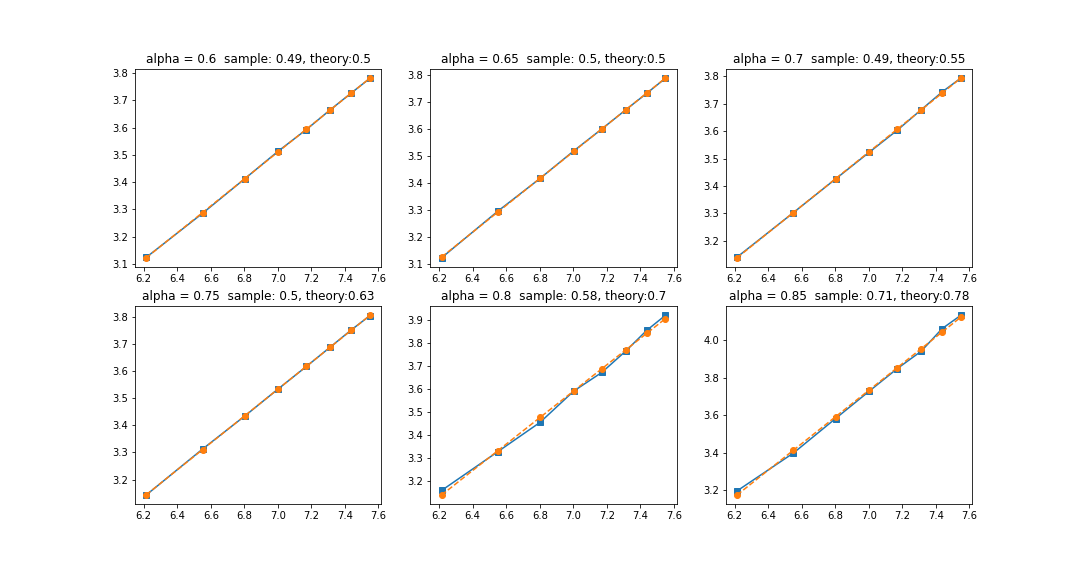}
    \caption{Caption: convergence of the numerator of Eq. \eqref{eq:dk_application}}
    \label{fig:theory_numerator}
\end{figure}

\end{document}